\documentclass[twoside,11pt]{article}

\usepackage[accepted]{melba}

\usepackage{graphicx}

\usepackage{wrapfig}
\usepackage{color, colortbl}
\usepackage{xcolor}
\definecolor{Gray}{gray}{0.9}

\usepackage{caption}
\usepackage{subcaption}
\usepackage{todonotes}

\usepackage{fancyhdr,amsmath,amssymb,amsfonts}
\usepackage[ruled,vlined]{algorithm2e}

\usepackage{multirow}
\usepackage{todonotes}
\usepackage{soul}
\setstcolor{red}

\usepackage{stfloats}

\usepackage{makecell}

\melbaheading{2022:012}{https://www.melba-journal.org/papers/2022:012.html}{2022}{1-36}{09/2021}{04/2022}{Balelli, Silva and Lorenzi}{Information Processing in Medical Imaging (IPMI) 2021}{Aasa Feragen, Stefan Sommer, Julia Schnabel, Mads Nielsen}

\ShortHeadings{Private Probabilistic Framework for Federated Heterogeneous Multi-View Datasets Variability}{Balelli, Silva and Lorenzi}
\firstpageno{1}

\title{A Differentially Private Probabilistic Framework for Modeling the Variability Across Federated Datasets of Heterogeneous Multi-View Observations}

\author{\name Irene Balelli \email irene.balelli@inria.fr \\  % start right after \author{, or there will be an extra space
	\addr Universit\'e C\^ote d'Azur, Inria Sophia Antipolis-M\'editeran\'ee, Epione Research Project, France
	\AND
	\name Santiago Silva \email santiago-smith.silva-rincon@inria.fr \\
	\addr Universit\'e C\^ote d'Azur, Inria Sophia Antipolis-M\'editeran\'ee, Epione Research Project, France
	\AND
	\name Marco Lorenzi \email marco.lorenzi@inria.fr \\
	\addr Universit\'e C\^ote d'Azur, Inria Sophia Antipolis-M\'editeran\'ee, Epione Research Project, France
}

\begin{document}

% top matter
\maketitle
% abstract
\begin{abstract}%   <- trailing '%' for backward compatibility of .sty file
We propose 
a novel federated learning 
paradigm to model data variability among heterogeneous clients in multi-centric studies. Our method is expressed through a hierarchical Bayesian latent variable model, where client-specific parameters are assumed to be realization from a global distribution at the master level, which is in turn estimated to account for data bias and variability across clients. We show that our framework can be effectively optimized through expectation maximization (EM) over latent master's distribution and clients' parameters.
We also introduce formal differential privacy (DP) guarantees compatibly with our EM optimization scheme. We tested our method on the analysis of multi-modal medical imaging data and clinical scores from distributed clinical datasets of patients affected by Alzheimer's disease. 
We demonstrate that our method is robust when data is distributed either in iid and non-iid manners, even when local parameters perturbation is included to provide DP guarantees. 
Moreover, the variability of data, views and centers can be quantified in an interpretable manner, while guaranteeing high-quality data reconstruction as compared to state-of-the-art autoencoding models and federated learning schemes.
%We demonstrate that our method is robust when data is distributed either in iid and non-iid manners, even when local parameters perturbation is included to provide DP guarantees. 
%In addition, the variability of data, views and centers can be quantified in an interpretable manner,
%Our approach allows to quantify the variability of data, views and centers,
% while guaranteeing high-quality data reconstruction as compared to the state-of-the-art autoencoding models and federated learning schemes. 
%
The code is available at~\url{https://gitlab.inria.fr/epione/federated-multi-views-ppca}.
\end{abstract}

% keywords
\section{Introduction}\label{Sec1}
The analysis of medical imaging datasets 
% for the study of 
% neurodegenerative diseases, 
requires the joint modeling of 
multiple \emph{views}  
(or modalities), such as clinical scores and multi-modal medical imaging data. For example, in dataset from neurological studies, views are generated through different medical imaging data acquisition processes, as for instance Magnetic Resonance Imaging (MRI) or Positron Emission Tomography (PET). 
Each view provides specific information about the pathology, and the joint analysis of all views is necessary to improve diagnosis, for the discovery of pathological relationships or for predicting  disease evolution. Nevertheless, the integration of \emph{multi-views} data, accounting for their mutual interactions and their joint variability, 
presents a number of challenges.

When dealing with high dimensional and noisy data it is crucial to be able to extract an informative lower dimensional representation to disentangle the relationships among observations, accounting for the intrinsic heterogeneity of the original complex data structure. From a statistical perspective, this implies the estimation of a model of the joint variability across views, or equivalently the development of a joint \emph{generative model}, assuming the existence of a common latent representation generating all views.

Several data assimilation methods based on dimensionality reduction have been developed~\citep{cunningham2015linear}, and successfully applied to a variety of domains. The main goal of these methods is to identify a suitable lower dimensional latent space, where meaningful statistical properties of the original dataset are identified after projection. The most basic among such methods is Principal Component Analysis (PCA)~\citep{jolliffe1986principal}, where data are projected over the axes of maximal variability. More flexible approaches based on non-linear representation of the data variability are Auto-Encoders~\citep{kramer1991nonlinear,goodfellow2016deep}, 
enabling  to  learn  a  low-dimensional representation  minimizing  the  reconstruction error. In the medical imaging community, non-linear counterparts of PCA have been also proposed by extending the notion of principal components and variability to the Riemannian setting \citep{sommer2010manifold,banerjee2017robust}.

In some cases, Bayesian counterparts of the original dimensionality reduction methods have been developed, such as 
Probabilistic Principal Component Analysis  (PPCA)~\citep{tipping1999probabilistic}, based on factor analysis,
or, more recently, Variational Auto-Encoders  (VAEs)~\citep{kingma2019introduction}, and Bayesian principal geodesic analysis~\citep{zhang2013probabilistic,hromatka2015hierarchical,fletcher2016probabilistic}. 
In particular, VAEs are machine learning algorithms 
based on a generative function which allows probabilistic data reconstruction from the latent space. Encoder and decoder can be flexibly parametrized by neural networks (NNs), and efficiently optimized through Stochastic Gradient Descent (SGD). The added values of Bayesian methods 
is to provide
a tool for sampling new observations from the estimated data distribution, and 
quantify the uncertainty of data and parameters. In addition, Bayesian model selection criteria, such as the Watanabe-Akaike Information Criteria (WAIC)~\citep{gelman2014understanding}, 
allow to perform automatic model selection. 

Multi-centric biomedical studies offer a great opportunity to 
significantly 
increase the quantity and quality of available data, hence to improve the statistical reliability of their analysis.
Nevertheless, in this context, three main data-related challenges should be considered. 1) \emph{Statistical heterogeneity} of \emph{local datasets} (\emph{i.e.} center-specific datasets): observations may be non-identically distributed across centers with respect to some characteristics affecting the output (\emph{e.g.} diagnosis). Additional variability in local datasets can also come from data collection and acquisition bias~\citep{kalter2019development}. 2) \emph{Missing views}: not all views are usually available for each center, due for example to heterogeneous data acquisition and processing pipelines. 3) \emph{Privacy} concerns: privacy-preserving laws are currently enforced to ensure protection of personal data (\emph{e.g.} the European General Data Protection Regulation - GDPR\footnote{\url{https://gdpr-info.eu/}}), often preventing the centralized analysis of data collected in multiple centers~\citep{iyengar2018healthcare,chassang2017impact}. 
These limitations impose the need for extending currently available data assimilation methods to handle decentralized heterogeneous data and missing views in local datasets.

Federated learning (FL) is an emerging analysis paradigm specifically developed for the decentralized training of machine learning models.
The standard aggregation method in FL is Federated Averaging (FedAvg)~\citep{mcmahan2017communication}, which combines locally trained models via weighted averaging. This aggregation scheme is generally sensitive to statistical heterogeneity, which naturally arises in federated datasets~\citep{li2020federated}, 
for example when dealing with multi-view data, or when data are not uniformly represented across data centers (e.g. non-iid distributed). In this case a faithful representation of the variability across centers is not guaranteed.

In order to guarantee data governance, FL methods are conceived to avoid sensitive data transfer among centers: raw data are processed within each center, and only local parameters are shared with the master.  
% Compared to traditional machine learning approaches, FL methods allow an improvement of privacy preservation of local datasets, since no data transfer is required to learn the final (global) model, but only statistics (such as gradients) are shared. 
Nevertheless, no formal privacy guarantees are provided on the shared statistics, which may still reveal sensitive information about individual data points used to train the model.  Differential privacy (DP) is an established framework to provide theoretical guarantees about the anonymity of the shared statistics with respect to the training data points. Recent works~\citep{abadi2016deep,geyer2017differentially,triastcyn2019federated} show the importance of combining FL and DP to prevent potential information leakage form the shared parameters, while providing theoretical privacy guarantees for both clients and server.

% \todo[inline]{Differential privacy}

\begin{figure}[ht]
\centering
\includegraphics[scale=0.75]{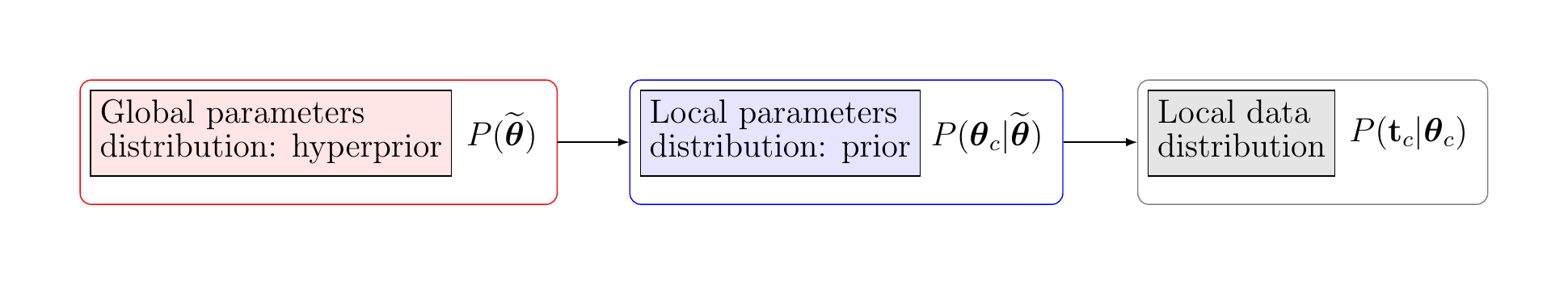}
\caption{Hierarchical structure of Fed-mv-PPCA. Global parameters $\widetilde{\boldsymbol{\theta}}$  characterize the distribution of the local  $\boldsymbol{\theta}_{c}$, which parametrize the local data distribution in each center.
}\label{Hierarc_model}
\end{figure}

We present here Federated multi-view PPCA (Fed-mv-PPCA), a novel FL framework for data assimilation of heterogeneous multi-view datasets. Our framework is designed to account for the heterogeneity of federated datasets through a fully Bayesian formulation. Fed-mv-PPCA is based on a hierarchical dependency of the model's parameters to handle different sources of variability in the federated dataset (Figure \ref{Hierarc_model}). The method is based on a linear generative model, assuming Gaussian latent variables and noise, and allows to account for missing views and observations across datasets. 
In practice, we assume that there exists an ideal global distribution of each parameter, from which local parameters are generated to account for the local data distribution for each center.  We show, in addition, that the privacy of the shared parameters of Fed-mv-PPCA can be explicitly quantified and guaranteed by means of DP. The code developed in Python is publicly available at \url{https://gitlab.inria.fr/epione/federated-multi-views-ppca}.

The paper is organized as follows: in Section \ref{Sec2} we provide a brief overview of the state-of-the-art and highlight the advancements introduced with Fed-mv-PPCA. In Section \ref{Sec3} we describe Fed-mv-PPCA, while its extension to improve privacy preservation through DP is provided in Section \ref{sec:diff_privacy}. In Section \ref{Sec4} we show results with applications to synthetic data and to data from the Alzheimer's Disease Neuroimaging Initiative dataset (ADNI). 
Section \ref{Sec5} concludes the paper with a brief discussion.

% \section{State of the art}\label{Sec2}
\section{Related Works}\label{Sec2}

%The method presented in this paper falls within two main categories: Bayesian methods for data assimilation, and FL methods for heterogeneous datasets.

Several methods for dimensionality reduction based on generative models have been developed in the past years, starting from the seminal work of PPCA  by \cite{tipping1999probabilistic}, to Bayesian Canonical Correlation Analysis (CCA)~\citep{klami2013bayesian}, which \cite{matsuura2018generalized} extended to include multiple views and missing modalities, up to more complex methods based on multi-variate association models~\citep{shen2019brain}, developed, for example, to integrate multi-modal brain imaging data and high-throughput genomics data. Other works with interesting applications to medical imaging data are based on Riemannian approaches to better deal with non linearity \citep{sommer2010manifold,banerjee2017robust}, and have been extended to a latent variable formulation (Probabilistic Principal Geodesic Analysis - PPGA - by \cite{zhang2013probabilistic}).

More recent methods for the probabilistic analysis of multi-views datasests include the multi channel Variational Autoencoder (mc-VAE) by  \cite{antelmi2019sparse} and Multi-Omics Factor Analysis (MOFA) by  \cite{Argelaguet2018}. MOFA generalizes PPCA for the analysis of multi-omics data types, supporting different noise models to adapt to continuous, binary and count data, while mc-VAE  extends the classic VAE~\citep{kingma2014stochastic} to jointly account for multi-views data. Additionally, mc-VAE can handle sparse datasets: data reconstruction in testing can be inferred from available views, if some are missing.

Despite the possibility offered by the above methods for performing data assimilation and integrating multiple views, these approaches have not been conceived to handle federated datasets.

Statistical heterogeneity is a key challenge in FL and, more generally, in multi-centric studies~\citep{li2020federated}. To tackle this problem, \cite{li2018federated} proposed the FedProx algorithm, which improves FedAvg by allowing for partial local work (\emph{i.e.} adapting the number of local epochs) and by introducing a proximal term to the local objective function to avoid divergence due to data heterogeneity.
Other methods have been developed under the Bayesian non-parametric formalism, 
such as probabilistic neural matching~\citep{yurochkin2019bayesian}, where the local parameters of NNs are federated depending on neurons similarities. 

Since the development of FedAvg,  researchers have been focusing in developing FL frameworks robust to the statistical heterogeneity across clients \citep{sattler2019robust,liang2020think}. Most of these frameworks are however formulated for training schemes based on stochastic gradient descent, with principal applications to NNs models. Nevertheless, beyond applications taylored around NNs, we still lack of a consistent and privacy-compliant Bayesian framework for the estimation of local and global data variability, as part of a global optimization model, while accounting for data heterogeneity.
In particular, 
%even if privacy is a major concern motivating the development of FL methods, classical works in
 FL alone does not provide clear theoretical guarantees for privacy preservation, leaving the door open to potential data leakage from malicious clients or the central server, such as through model inversion~\citep{fredrikson2015model}, and researchers are currently focusing to adapt FL schemes to account for DP mechanisms
% . A solution to this problem can be provided by coupling FL schemes with DP
 ~\citep{mcmahan2017learning}.

All these considerations ultimately motivate for the development of Fed-mv-PPCA and its differential private extension. The main contributions of the work presented in this paper are the following:
\begin{itemize}
    \item we theoretically develop a novel Bayesian hierarchical framework, Fed-mv-PPCA, for data assimilation from heterogeneous multi-views private federated datasets;
    \item we investigate the improvement our framework's security against data leakage by coupling it with differential privacy, and propose DP-Fed-mv-PPCA; %Theoretical privacy guarantees are ensured in this context;
    \item We apply both models to synthetic data and real multi-modal imaging data and clinical scores form the Alzheimer's Disease Neuroimaging Initiative, demonstrating the robustness of our framework against non-iid data distribution across centers and missing modalities.
\end{itemize}
% This provides us motivation for the development of Fed-mv-PPCA and DP-Fed-mv-PPCA, a Bayesian framework for data assimilation from heterogeneous multi-views private federated datasets.
% {\color{red}, and its improved differentially private scheme}.

\section{Methods}\label{Sec3}

\subsection{Federated multi-views PPCA}

\subsubsection{Problem setup}

We consider $C$ independent centers. Each center $c\in\{1,\dots,C\}$ owns a private local dataset $T_{c} = \left\{\mathbf{t}_{c,n}\right\}_{n=1,\dots,N_c}$, where we denote by  $\mathbf{t}_{c,n}$ the data row for subject $n$ in center $c$, with $n=1,\dots,N_c$. We assume that a total of $K$ distinct views have been measured across all centers, and we allow missing views in some local dataset (\emph{i.e.} some local dataset could be incomplete, including only measurements for $K_{c}< K$ views). For every $k\in\{1,\dots,K\}$, the dimension of the $k^{\textrm{th}}$-view (\emph{i.e.} the number of features defining the $k^{\textrm{th}}$-view) is $d_k$, and we define $d:=\sum_{k=1}^K d_k$. We denote by $\mathbf{t}_{c,n}^{(k)}$ the raw data of subject $n$ in center $c$ corresponding to the $k^{\textrm{th}}$-view, hence $\mathbf{t}_{c,n} = \left(\mathbf{t}_{c,n}^{(1)},\dots,\mathbf{t}_{c,n}^{(K)}\right)$.

\subsubsection{Modeling assumptions}\label{sec:meth}

The main assumption at the basis of Fed-mv-PPCA is the existence of a hierarchical structure underlying the data distribution. In particular, we assume that there exist global parameters $\widetilde{\boldsymbol{\theta}}$, following a distribution $P(\widetilde{\boldsymbol{\theta}})$, able to describe the global data variability, \emph{i.e.} the ensemble of local datasets. For each center, local parameters $\boldsymbol{\theta}_{c}$ are generated from $P(\boldsymbol{\theta}_{c}|\widetilde{\boldsymbol{\theta}})$, to account for the specific variability of the local dataset. Finally, local data $\mathbf{t}_{c}$ are obtained from their local distribution $P(\mathbf{t}_{c}|\boldsymbol{\theta}_{c})$. Given the federated datasets, Fed-mv-PPCA provides a consistent Bayesian framework to solve the inverse problem and estimate the model's parameters across the entire hierarchy. 

We assume that in each center $c$, the local data  of subject $n$ corresponding to the $k^{\textrm{th}}$-view, $\mathbf{t}_{c,n}^{(k)}$, follows the generative model:
\begin{equation}\label{eq_tcnkg}
    \mathbf{t}_{c,n}^{(k)}=W_{c}^{(k)}\mathbf{x}_{c,n}+\boldsymbol{\mu}_{c}^{(k)}+\boldsymbol{\varepsilon}_{c}^{(k)},
\end{equation}
where 
$\mathbf{x}_{c,n}\sim\mathcal{N}(0,\mathbb{I}_q)$ is a $q$-dimensional latent variable, and $q<\min_k(d_k)$ is the dimension of the latent-space.
$W_{c}^{(k)}\in\mathbb{R}^{d_k\times q}$ provides the linear mapping between latent space and observations for the $k^{\textrm{th}}$-view,
$\boldsymbol{\mu}_{c}^{(k)}\in\mathbb{R}^{d_k}$ is the offset of the data corresponding to view $k$, and 
$\boldsymbol{\varepsilon}_{c}^{(k)}\sim\mathcal{N}\left(0,{\sigma_{c}^{(k)}}^2\mathbb{I}_{d_k}\right)$ is the Gaussian noise for the $k^{\textrm{th}}$-view. 
This formulation induces a Gaussian distribution over $\mathbf{t}_{c,n}^{(k)}$, implying:
\begin{equation}\label{dist_tcnkg}
\mathbf{t}_{c,n}^{(k)}\sim\mathcal{N}(\boldsymbol{\mu}_{c}^{(k)},C_{c}^{(k)}),
\end{equation}
where $C_{c}^{(k)}=W_{c}^{(k)}{W_{c}^{(k)}}^T+{\sigma_{c}^{(k)}}^2\mathbb{I}_{d_k}\in\mathbb{R}^{d_k\times d_k}$. 
Finally, a compact formulation for $\mathbf{t}_{c,n}$ (\emph{i.e.} considering all views concatenated) can be derived from Equation \eqref{eq_tcnkg}:
\begin{equation}\label{eq_tcng}
    \mathbf{t}_{c,n}=W_{c}\mathbf{x}_{c,n}+\boldsymbol{\mu}_{c}+\Psi_{c},
\end{equation}
where $W_{c}, \boldsymbol{\mu}_{c}$ are obtained by concatenating all $W_{c}^{(k)}, \boldsymbol{\mu}_{c}^{(k)}$, and $\Psi_{c}$ is a block diagonal matrix, where the $k^{\textrm{th}}$-block is given by $\boldsymbol{\varepsilon}_{c}^{(k)}$.  
The local parameters describing the center-specific dataset thus are $\boldsymbol{\theta}_{c}:=\left\{\boldsymbol{\mu}_{c}^{(k)},W_{c}^{(k)},{\sigma_{c}^{(k)}}^2\right\}_{k}$. According to our hierarchical formulation, we assume that each local parameter in $\boldsymbol{\theta}_{c}$ is a realization of a common global prior distribution described by $\widetilde{\boldsymbol{\theta}}:=\left\{\widetilde{\boldsymbol{\mu}}^{(k)},\sigma_{\widetilde{\boldsymbol{\mu}}^{(k)}},\widetilde{W}^{(k)},\sigma_{\widetilde{W}^{(k)}},\widetilde{\alpha}^{(k)},\widetilde{\beta}^{(k)}\right\}_{k}$. In particular we assume that $\boldsymbol{\mu}_{c}^{(k)}$ and $W_{c}^{(k)}$ are normally distributed, while the variance of the Gaussian error, ${\sigma_{c}^{(k)}}^2$, follows an inverse-gamma distribution. Formally:
\begin{eqnarray}
\boldsymbol{\mu}_{c}^{(k)}|\widetilde{\boldsymbol{\mu}}^{(k)},\sigma_{\widetilde{\boldsymbol{\mu}}^{(k)}} & \sim & \mathcal{N}\left(\widetilde{\boldsymbol{\mu}}^{(k)},\sigma_{\widetilde{\boldsymbol{\mu}}^{(k)}}^2\mathbb{I}_{d_k}\right),\label{eq:priors1}\\
W_{c}^{(k)}|\widetilde{W}^{(k)},\sigma_{\widetilde{W}^{(k)}} & \sim & \mathcal{MN}_{k,q}\left(\widetilde{W}^{(k)},\mathbb{I}_{d_k},\sigma_{\widetilde{W}^{(k)}}^2\mathbb{I}_q\right), \label{eq:priors2}\\
{\sigma_{c}^{(k)}}^2|\widetilde{\alpha}^{(k)},\widetilde{\beta}^{(k)} & \sim &  \textrm{Inverse-Gamma}(\widetilde{\alpha}^{(k)},\widetilde{\beta}^{(k)}), \label{eq:priors4}
\end{eqnarray}
where $\mathcal{MN}_{k,q}$ denotes the matrix normal distribution of dimension $d_k\times q$.

\subsubsection{Proposed framework}\label{ssec:framework}

The assumptions made in Section \ref{sec:meth} allow to naturally define an optimization scheme based on Expectation Maximization (EM) locally, and on Maximum Likelihood estimation (ML) at the master level (Algorithm \ref{algo}). Figure~\ref{Graph_model} shows the graphical model of Fed-mv-PPCA. 

\begin{figure}[ht]
\centering
\includegraphics[scale=0.6]{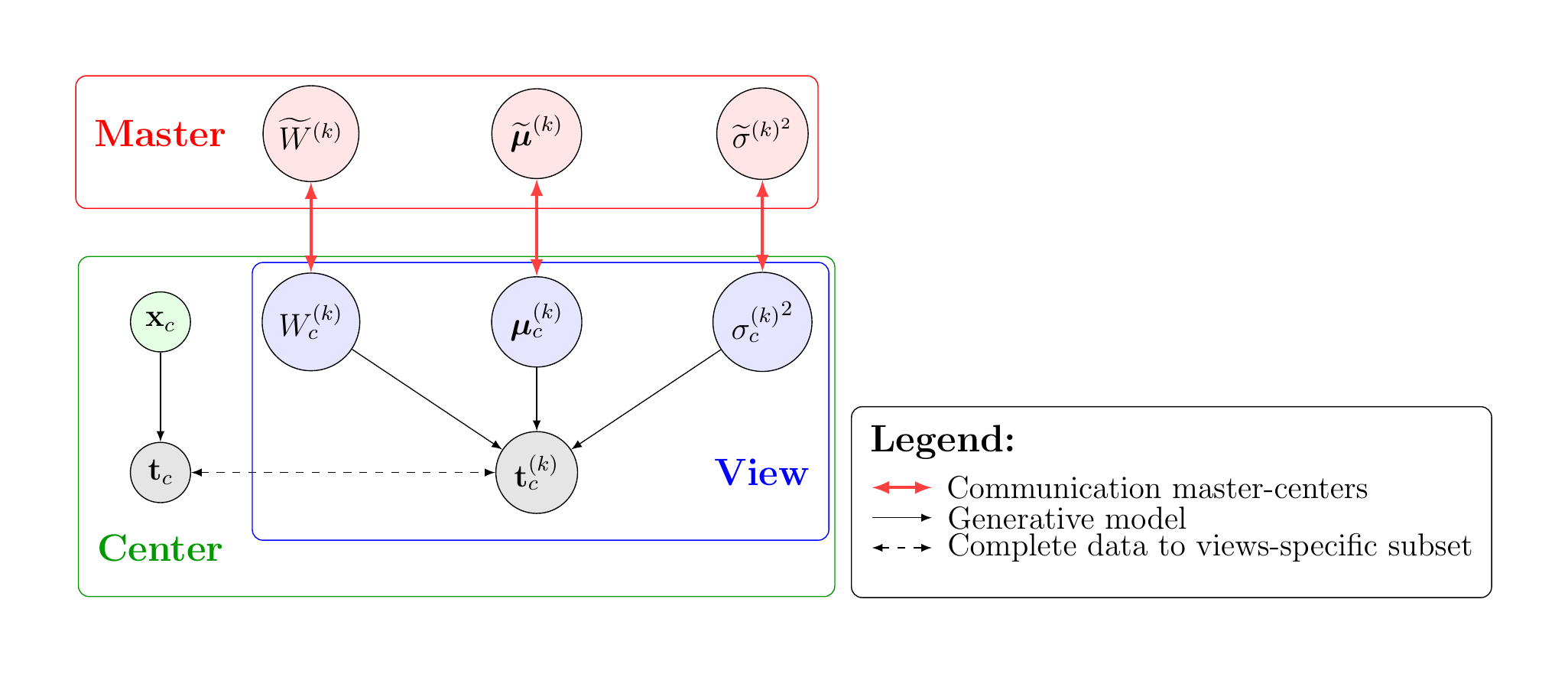}
\caption{Graphical model of Fed-mv-PPCA. Thick double-sided red arrows relate nodes which are shared between center and master, while plain black arrows define the relations between the local dataset and the generative model parameters. Grey filled circles correspond to raw data: the dashed double-sided arrow highlights the complexity of the dataset, composed by multiple views. }\label{Graph_model}
\end{figure}

\begin{algorithm}[H]
\SetKwInOut{Input}{Input}\SetKwInOut{Output}{Output}
\SetAlgoLined
\Input{Rounds $R$; Iterations $I$; Latent space dimension $q$}
\Output{Global parameters $\widetilde{\boldsymbol{\theta}}$}
\BlankLine
 \For{$r=1,\dots,R$}{
 \For{$c=1,\dots,C$ \textbf{in parallel}}{
  Each center $c$ initializes $\boldsymbol{\theta}_{c}$ using $P(\boldsymbol{\theta}_{c}|\widetilde{\boldsymbol{\theta}})$\;
  $I$ iterations of MAP estimation of $\boldsymbol{\theta}_{c}$ using $\widetilde{\boldsymbol{\theta}}$ as prior\;
%   \eIf{$r=1$}{
%   Each center $c$ initializes randomly local parameters $\boldsymbol{\theta}_{c}$\;
%   $I$ iterations of EM estimation to optimize $\boldsymbol{\theta}_{c}$\;
%   }{
%   Each center initializes $\boldsymbol{\theta}_{c}$ using $P(\boldsymbol{\theta}_{c}|\widetilde{\boldsymbol{\theta}})$\;
%   $I$ iterations of MAP estimation (EM + prior) to optimize $\boldsymbol{\theta}_{c}$ using $\widetilde{\boldsymbol{\theta}}$ as prior\;
  }
  Each center $c$ returns $\boldsymbol{\theta}_{c}$ to the master\;
  The master collects $\boldsymbol{\theta}_{c}$, $c=1,\dots,C$ and estimates $\widetilde{\boldsymbol{\theta}}$ through ML\;
  The master sends $\widetilde{\boldsymbol{\theta}}$ to all centers
 }
 \caption{Fed-mv-PPCA algorithm}\label{algo}
\end{algorithm}

With reference to Algorithm \ref{algo}, the optimization of Fed-mv-PPCA is as follows: 

\paragraph{Optimization.}
The master collects the local parameters $\boldsymbol{\theta}_{c}$ for $c\in\{1,\dots,C\}$ and estimates the ML updated global parameters characterizing the prior distributions of Equations \eqref{eq:priors1} to \eqref{eq:priors4}.
Updated global parameters $\widetilde{\boldsymbol{\theta}}$ are returned to each center, and serve as priors to update the MAP estimation of the local parameters $\boldsymbol{\theta}_{c}$, through the M step on the functional $\mathbf{E}_{p(\mathbf{x}_{c,n}|\mathbf{t}_{c,n})}\ln{\left(p(\mathbf t_{c,n},\mathbf{x}_{c,n}|\boldsymbol{\theta}_{c})p(\boldsymbol{\theta}_{c}|\widetilde{\boldsymbol{\theta}})\right)}$, where:

\begin{eqnarray*}\mathbf{x}_{c,n}|\mathbf{t}_{c,n}\sim\mathcal{N}\left(\Sigma_{c}^{-1}{W_{c}}^T\Psi_{c}^{-1}(\mathbf{t}_{c,n}-\boldsymbol{\mu}_{c}),\Sigma_{c}^{-1}\right), \Sigma_{c}:=(\mathbb{I}_q+W_{c}^T\Psi_{c}^{-1}W_{c})
    \end{eqnarray*} 
and 
\begin{eqnarray*}
\langle \ln{\left(p(\mathbf t_{c,n},\mathbf{x}_{c,n}|\boldsymbol\theta_c)\right)}\rangle & = & -\sum_{n=1}^{N_{c}}\left\{\sum_{k=1}^K\left[\frac{d_k}{2}\ln{\left({\sigma_{c}^{(k)}}^2\right)}+\frac{1}{2{\sigma_{c}^{(k)}}^2}    \|\mathbf{t}_{c,n}^{(k)}-\boldsymbol{\mu}_{c}^{(k)}\|^2+\right.\right.  \\
& & \frac{1}{2{\sigma_{c}^{(k)}}^2}tr\left({W_{c}^{(k)}}^T{W_{c}^{(k)}}\langle \mathbf{x}_{c,n}\mathbf{x}_{c,n}^T\rangle\right)   \\
& & \left.-\frac{1}{{\sigma_{c}^{(k)}}^2}\langle \mathbf{x}_{c,n}\rangle^T{W_{c}^{(k)}}^T\left(\mathbf{t}_{c,i}^{(k),g}-\boldsymbol{\mu}_{c}^{(k)}\right)\right] \left.+\frac{1}{2}tr\left(\langle \mathbf{x}_{c,n}\mathbf{x}_{c,n}^T\rangle\right)\right\},
\end{eqnarray*}

\paragraph{Initialization at round r=1.}
The latent-space dimension $q$, the number of local iterations $I$ and the number of communication rounds $R$ (\emph{i.e.} number of complete cycles centers-master) are user-defined parameters. 
For the sake of simplicity, we set here the same number of local iterations for every center. Note that this constraint can be easily adapted to take into account systems heterogeneity among centers, as well as the size of each local dataset. 
At the first round, local parameters initialization, hence optimization, can be performed in two distinct ways: 1) each center  can initialize randomly every local parameter, then perform EM through $I$ iterations, maximizing the functional $\langle \ln{\left(p(\mathbf t_{c,n},\mathbf{x}_{c,n}|\boldsymbol\theta_c)\right)}\rangle$; 2) the master can provide priors for at least some parameters, which will be optimized using MAP estimation as described above. In case of a random initialization of local parameters, the number of EM iterations for the first round can be increased: this can be seen as an exploratory phase.

The reader can refer to Appendix A for further details on the theoretical formulation of Fed-mv-PPCA and the corresponding optimization scheme.
% \todo[inline]{Question about the organization of the paper: to be kept like this or rather put here the differential private section, then results comparing performances with and without DP + comparison with VAE/mcVAE?}

\subsection{Fed-mv-PPCA with Differential Privacy}\label{sec:diff_privacy}

% In this section we focus on quantifying and improving the privacy of our framework. 
Despite the Bayesian federated learning scheme deployed prevents data transfer, it does not provide theoretical privacy guarantees on the shared statistics. Differential privacy (DP)~\citep{dwork2014algorithmic,abadi2016deep} is a standard framework for privacy-preserving computations allowing to quantify a privacy protection budget attached to a given operation, and to sanitize model parameters through output perturbation machanisms based on the addition of a random noise. The noise strength has to be tuned to ensure a good balance between privacy and utility of the outputs.% (such as, in our case, reconstruction and accuracy). 

In Section \ref{sec:DP1} we recall the standard definition of differential privacy and established results on classical random perturbation mechanisms, as well as the composition theorem \citep{dwork2014algorithmic}. A differentially private version of Fed-mv-PPCA (DP-Fed-mv-PPCA) is subsequently derived in Section \ref{sec:DP2}.%, while Section \ref{sec:dp_res} provides results on the ADNI dataset.

\subsubsection{Differential privacy: background %Definitions and results
}\label{sec:DP1}

We denote by $D, D'$ two datasets: $D$ and $D'$ are said to be neighboring or adjacent datasets if they only differ by a datapoint $\mathbf{t}'$, $D=D'\cup\{\mathbf{t}'\}$. In this case we write $\|D-D'\|=1$, where $\|\cdot\|$ denotes the cardinality of a given set.

\begin{definition}\label{def1}
A randomized algorithm $\mathcal{M}:\mathcal{D}\to\mathcal{R}$ with domain $\mathcal{D}$ and range $\mathcal{R}$ is $(\varepsilon,\delta)$-differentially private if for any $D,D'\in\mathcal{D}$ s.t. $\|D-D'\|=1$ and for any  $\mathcal{S}\in\mathcal{R}$:
\begin{equation*}
    \mathbb{P}\left[\mathcal{M}(D)\in\mathcal{S}\right]\leq e^{\varepsilon}\left[\mathcal{M}(D')\in\mathcal{S}\right]+\delta
\end{equation*}
\end{definition}
When $\delta=0$, we simply say that the algorithm $\mathcal{M}$ is $\varepsilon$-differentially private.

A common mechanism to approximate a deterministic function or a query $f:\mathcal{D}\to\mathbb{R}^d$ with differential privacy is the addition of a random noise calibrated on the sensitivity of $f$.

\begin{definition}
The $l_p$-sensitivity of a function $f:\mathcal{D}\to\mathbb{R}^d$ is defined as:
\begin{equation*}
    \Delta_p f = \max_{\|D-D'\|=1}\|f(D)-f(D')\|_p
\end{equation*}
\end{definition}

Classical mechanisms used for perturbation are the \emph{Laplace mechanism} and the \emph{Gaussian mechanism}. A Laplace (resp. Gaussian) mechanism is simply obtained by computing $f$, hence perturbing it with noise added from a Laplace (resp. Gaussian) distribution centered in the origin and with variance depending on the sensitivity of $f$:
\begin{equation*}
    \mathcal{M}(D):=f(D)+Noise,
\end{equation*}
where $Noise\sim\textrm{Laplace}\left(0,\textrm{std}_L(\Delta_p f)\right)$ (resp. $Noise\sim\mathcal{N}\left(0,\textrm{var}_G(\Delta_p f)\right)$).

Hereafter we recall the condition of a Laplace (resp. Gaussian) mechanism to preserve $(\varepsilon,\delta)$-DP.

\begin{theorem}\label{thm1}
Given any function $f:\mathcal{D}\to\mathbb{R}^d$ and $\varepsilon>0$, the Laplace mechanism defined as
\begin{equation*}
    \mathcal{M}(D):=f(D)+(L_1,\dots,L_d),
\end{equation*}
where $L_i$ are iid drawn from $\textrm{Laplace}(0,\Delta_1 f/\varepsilon)$, preserves $\varepsilon$-DP.
\end{theorem}

\begin{theorem}\label{thm2}
Given any function $f:\mathcal{D}\to\mathbb{R}^d$ and $(\varepsilon,\delta)\in(0,1)^2$, the Gaussian mechanism defined as
\begin{equation*}
    \mathcal{M}(D):=f(D)+\mathcal{N}\left(\mathbf{0},\left(\frac{\sqrt{2\ln{(1.25\delta)}}\Delta_2 f}{\varepsilon}\right)^2\mathbb{I}_d\right),
\end{equation*}
preserves $(\varepsilon,\delta)$-DP.
\end{theorem}

The formal proofs of Theorems \ref{thm1}-\ref{thm2} are provided \emph{e.g.} by \cite{dwork2014algorithmic}.

An improved Gaussian mechanism is further described by \cite{zhao2019reviewing}, with the advantages of 1) remaining valid for $\varepsilon>1$ given $\delta\leq0.5$, and 2) adding a smaller noise as compared to the result of Theorem \ref{thm2} in the case $0<\varepsilon\leq1$. 

\begin{theorem}\label{thm2+}
Given any function $f:\mathcal{D}\to\mathbb{R}^d$, $\varepsilon>0$, and $\delta\in(0,0.5)$, the Gaussian mechanism defined as
\begin{equation*}
    \mathcal{M}(D):=f(D)+\mathcal{N}\left(\mathbf{0},\left(\frac{(c+\sqrt{c^2+\varepsilon})\Delta_2 f}{\varepsilon\sqrt{2}}\right)^2\mathbb{I}_d\right),
\end{equation*}
where $c=\sqrt{\ln{\left(2/(\sqrt{16\delta+1}-1)\right)}}$, preserves $(\varepsilon,\delta)$-DP.
\end{theorem}

It is worth noting that Theorems \ref{thm2} and \ref{thm2+} can be naturally extended to queries mapping to $\mathbb{R}^{d\times q}$ and matrix normal mechanisms:

\begin{corollary}\label{cor_thm2}
Given any function $f:\mathcal{D}\to\mathbb{R}^{d\times q}$, $\varepsilon>0$, and $\delta\in(0,0.5)$, the matrix normal mechanism defined as
\begin{equation*}
    \mathcal{M}(D):=f(D)+\mathcal{MN}_{d,q}\left(\mathbf{0}_{d,q},\mathbb{I}_d,\left(\frac{(c+\sqrt{c^2+\varepsilon})\Delta_2 f}{\varepsilon\sqrt{2}}\right)^2\mathbb{I}_q\right),
\end{equation*}
where $c=\sqrt{\ln{\left(2/(\sqrt{16\delta+1}-1)\right)}}$, preserves $(\varepsilon,\delta)$-DP.
\end{corollary}

%\begin{proof}
%The proof of Corollary \ref{cor_thm2} follows form Theorem \ref{thm2+} and the definition of matrix normal distribution, stating that $X\sim\mathcal{MN}_{d,q}(M,U,V)$ if and only if $\mathrm{vec}(X)\sim\mathcal{N}_{dq}(\mathrm{vec}(M),V\otimes U)$, where $\mathrm{vec}(M)$ denotes the vectorization of $M$ and $\otimes$ denotes the Kronecker product.
%\end{proof}

We conclude this section by recalling the well known composition theorem \citep{dwork2014algorithmic}, which will be useful to quantify the global privacy budget for each center in the next sections.

\begin{theorem}\label{thm3}
For $i=1,\dots,k$, let $\mathcal{M}_i:\mathcal{D}\to\mathcal{R}_i$ be an $(\varepsilon_i,\delta_i)$-differentially private algorithm, and $\mathcal{M}:\mathcal{D}\to\prod_{i=1}^k\mathcal{R}_i$ defined as $\mathcal{M}(\mathcal{D}):=(\mathcal{M}_1(\mathcal{D}),\dots, \mathcal{M}_k(\mathcal{D}))$. Then $\mathcal{M}$ is $\left(\sum_{i=1}^k\varepsilon_i,\sum_{i=1}^k\delta_i\right)$-differentially private.
\end{theorem}

\subsubsection{Differential privacy for local parameters}\label{sec:DP2}

In this section we 
% discuss the sensitivity of each local parameter update and 
propose a novel federated learning scheme for Fed-mv-PPCA with DP to protect client-level privacy and avoid potential private information leakage from the shared local parameters.

We are interested in preserving the privacy of the shared local parameters $\boldsymbol{\theta}_c=\{\boldsymbol{\mu}_c^{(k)},W_c^{(k)},{\sigma_c^{(k)}}^2\}_{k}$, which can be done by the addition of some properly tuned random noise, as detailed in Section \ref{sec:DP1}. Nevertheless, the client-level optimization scheme in Fed-mv-PPCA is based on an iterative algorithm: therefore we do not have a closed formula to evaluate the sensitivity of each local parameter (\emph{i.e.} the queries), nor an upper bound. To overcome this problem, we propose to perform difference clipping~\citep{geyer2017differentially,zhang2021understanding}, one of the clipping strategies proposed for differentially private SGD models.
Algorithm \ref{algo2} outlines the optimization scheme for the DP-Fed-mv-PPCA framework.\\

\begin{algorithm}[H]
\SetKwInOut{Input}{Input}\SetKwInOut{Output}{Output}
\SetAlgoLined
\Input{Rounds $R$; Iterations $I$; Latent space dimension $q$; Privacy parameters $\varepsilon$, $\delta$}
\Output{Global parameters $\widetilde{\boldsymbol{\theta}}$}
\BlankLine
 \For{$r=1,\dots,R$}{
 \For{$c=1,\dots,C$ \textbf{in parallel}}{
  Initialize $\boldsymbol{\theta}_{c}$ using $P(\boldsymbol{\theta}_{c}|\widetilde{\boldsymbol{\theta}}[r-1])$\;
  Update local parameters: $I$ iterations of MAP estimation (EM + prior) to optimize $\boldsymbol{\theta}_{c}[r]$ using $\widetilde{\boldsymbol{\theta}}[r-1]$ as prior\;
  Compute difference: $\Delta\boldsymbol{\theta}_c[r]:=(\boldsymbol{\theta}_c[r]-\widetilde{\boldsymbol{\theta}}[r-1])$\;
  Clip: $\overline{\Delta\boldsymbol{\theta}}_c[r]:=\Delta\boldsymbol{\theta}_c[r]/\max{\left(1,\|\Delta\boldsymbol{\theta}_c[r]\|_p/g(\sigma_{\widetilde{\boldsymbol{\theta}}}[r-1])\right)}$\;
  Perturb: $\mathcal{M}_{\boldsymbol{\theta}_c}[r]:=\overline{\Delta\boldsymbol{\theta}}_c[r]+Noise(2g(\sigma_{\widetilde{\boldsymbol{\theta}}}[r-1]),\varepsilon,\delta)$\;
  Return $\overline{\boldsymbol{\theta}}_c[r]:=\mathcal{M}_{\boldsymbol{\theta}_c}[r]+\widetilde{\boldsymbol{\theta}}[r-1]$ to the master\;
%   Evaluate $\Delta(\boldsymbol{\theta}_c)$;
  
%   \eIf{$\Delta(\boldsymbol{\theta}_c)\leq S$}{
%   Return $\boldsymbol{\theta}_{c}$ to the master\;
%   }{
%   Return $\mathcal{M}:=\boldsymbol{\theta}_{c}+Noise(\Delta(\boldsymbol{\theta}_c),\varepsilon,\delta)$ to the master\;
%   }
  }
  The master collects all $\overline{\boldsymbol{\theta}}_c[r]$ and estimates $\widetilde{\boldsymbol{\theta}}[r]$ through ML\;
  The master sends $\widetilde{\boldsymbol{\theta}}[r]$ to all centers
 }
 \caption{DP-Fed-mv-PPCA algorithm}\label{algo2}
\end{algorithm}

% With respect to Algorithm \ref{algo}, two additional operations are performed at the client level:

\paragraph{Difference clipping and perturbation.}
With respect to Algorithm \ref{algo}, difference clipping and perturbation are performed at the client level compatibly with the probabilistic formulation of the model:
\begin{enumerate}
\item The client computes the difference between the current local update and the initial prior (\emph{i.e.} the corresponding global parameter obtained at the previous communication round, $r-1$): 
\begin{equation*}
\Delta\boldsymbol{\theta}_c[r]:=(\boldsymbol{\theta}_c[r]-\widetilde{\boldsymbol{\theta}}[r-1])
\end{equation*}
\item The updated difference is clipped according to the standard deviation of the prior:
\begin{equation*}
\overline{\Delta\boldsymbol{\theta}}_c[r]:=\Delta\boldsymbol{\theta}_c[r]\cdot\left(\max{\left(1,\frac{\|\Delta\boldsymbol{\theta}_c[r]\|_p}{g(\sigma_{\widetilde{\boldsymbol{\theta}}}[r-1])}\right)}\right)^{-1},
\end{equation*}
where $g(\sigma_{\widetilde{\boldsymbol{\theta}}}[r-1]):=\textrm{const}\cdot(\sigma_{\widetilde{\boldsymbol{\theta}}}[r-1])$, and the multiplicative constant is fixed by the user. This clipping mechanism enforces the $l_p$ norm of $\Delta\boldsymbol{\theta}_c[r]$ to be at most $g(\sigma_{\widetilde{\boldsymbol{\theta}}}[r-1])$. Consequently, the $l_p$ sensitivity of $\overline{\Delta\boldsymbol{\theta}_c}[r]$ is bounded by $2\cdot g(\sigma_{\widetilde{\boldsymbol{\theta}}}[r-1])$.
\item The clipped difference is perturbed:
\begin{equation*}
\mathcal{M}_{\boldsymbol{\theta}_c}[r]:=\overline{\Delta\boldsymbol{\theta}}_c[r]+Noise(2g(\sigma_{\widetilde{\boldsymbol{\theta}}}[r-1]),\varepsilon,\delta)
\end{equation*}
In particular, for $\overline{\Delta\boldsymbol{\mu}^{(k)}}_c$ and $\overline{\Delta W^{(k)}}_c$, we propose to use a Gaussian (resp. matrix normal) mechanism (Theorem \ref{thm2+}, resp. Corollary \ref{cor_thm2}), in accordance with the Gaussian prior distributions of these parameters, while a Laplace mechanism (Theorem \ref{thm1}) is used to perturb $\overline{\Delta{\sigma^{(k)}}^2}_c$.
\item The client adds again the prior and finally sends to the master $\overline{\boldsymbol{\theta}}_c[r]:=\mathcal{M}_{\boldsymbol{\theta}_c}[r]+\widetilde{\boldsymbol{\theta}}[r-1]$.
\end{enumerate}
Conversely to model clipping~\citep{abadi2016deep,wei2020federated}, where the parameter update is directly clipped and perturbed, difference clipping has the advantage to allow reducing the magnitude of the perturbation: indeed, we expect the $l_p$ norm of the difference $\Delta\boldsymbol{\theta}_c$ to be small compared to the $l_p$ norm of $\boldsymbol{\theta}_c$. 
Moreover, our framework provides a natural way to define the clipping parameter according to the prior. Indeed, the clipping parameter is defined here as the standard deviation of global parameters. Hence, from a conceptual viewpoint, we are enforcing local parameters updates to remain closer to the global ones by some ratio of their standard deviation. 
%On one hand, this implies a reduction of the ability of the framework in capturing the between-centers variability. On the other hand, t
This allows to obfuscate the  participation of the individual centers at the expense of a reduction of the ability of the framework in capturing the between-centers variability.

\paragraph{Privacy budget} 
% We can quantify an upper bound for the total privacy budget.

%\begin{corollary}\label{cor1}
%For the sake of simplicity, let us choose the same $\varepsilon, \delta$ for all mechanisms considered above (a generalization to a parameter-specific choice of $\varepsilon_i, \delta_i$ is straightforward). At each round, the total privacy budget for the local optimization algorithm in center $c$ is at most $(3K\varepsilon,2K\delta)$.
%\end{corollary}

\begin{theorem}\label{thm4}
For  sake of simplicity, let us choose the same $\varepsilon, \delta$ for all mechanisms considered above (a generalization to a parameter-specific choice of $\varepsilon_i, \delta_i$ is straightforward). The total privacy budget for the outputs of Algorithm \ref{algo2} is 
% at most 
$(3K\varepsilon,2K\delta)$, where $K$ is the total number of views.
\end{theorem}

\begin{proof}
The proof of Theorem \ref{thm4} follows from Theorems \ref{thm1}-\ref{thm2+} and Corollary \ref{cor_thm2}, 
% results on differentially private PCA~\citep{imtiaz2018differentially}
and by noting that data in each center are disjoint. 
% In the worst scenario all parameters are perturbed, hence, i
In all centers, we are dealing with the mechanism $\mathcal{M}:=(\mathcal{M}_{\boldsymbol{\mu}_c^{(k)}},\mathcal{M}_{W_c^{(k)}},\mathcal{M}_{{\sigma_c^{(k)}}^2})_k$, where for all $k$, $\mathcal{M}_{\boldsymbol{\mu}_c^{(k)}}$ and $\mathcal{M}_{W_c^{(k)}}$ are $(\varepsilon,\delta)$-differentially private, while for all $k$, ${M}_{{\sigma_c^{(k)}}^2}$ is $\varepsilon$-differentially private. The result follows thanks to composition Theorem \ref{thm3} and the invariance of differential privacy under post-processing.
\end{proof}

% \begin{remark}
% According to an observed decreasing standard deviations of global parameters, one can propose a differentially private optimization scheme for Fed-mv-PPCA with privacy parameters $(\varepsilon,\delta)$ decreasing during training, insuring improved privacy guarantees of the final global parameters.
% \end{remark}

\begin{corollary}
If for local parameter $\theta_c\in\boldsymbol{\theta}_c$ the client-specific differential parameters are $(\varepsilon_c,\delta_c)$, then the total privacy budget for the corresponding global parameter $\widetilde{\theta}$ is bounded by $(\max{(\varepsilon_c)},\max{(\delta_c)})$.
\end{corollary}

\begin{proof}
The result directly follows from Theorem \ref{thm4} and by considering Definition \ref{def1} and the monotonicity of the exponential function.
\end{proof}

\subsection{Computational complexity and communication cost}\label{subsec:complexity}

The computational complexity of local parameters optimization in Fed-mv-PPCA (with or without the introduction of the DP mechanism depicted in Section \ref{sec:diff_privacy}) can be derived from the complexity of standard PPCA \citep{chen2009robust}. We recall that performing simple PPCA locally in center $c\in\{1,\dots,C\}$ implies a computational complexity of $\mathcal{O}(N_cd_cq)$, where $N_c$ and $d_c$ are respectively the number of samples and dimensions in center $c$, while $q$ is the chosen latent dimension. In the multi-view extension here considered, the total dimension is decomposed across views, meaning that $d_c:=\sum_{k\in K_c}d_k$, where $K_c$ is the set of observed views in center $c$, and $d_k$ the dimension of view $k$. The complete data log-likelihood to be maximized in the M step of the expectation-maximization algorithm, can consequently be written as a sum over the number of samples and number of observed views in center $c$ (see Section \ref{ssec:framework} and Appendix A.). For each $k\in K_c$, the computational complexity to optimize all $k$-specific local parameters is $\mathcal{O}(N_cd_kq)$. This finally implies a computational complexity of $\mathcal{O}(N_cq\prod_{k\in K_c}d_k)\approx\mathcal{O}(N_c\prod_{k\in K_c}d_k)$ when $q\ll\min_k(d_k)$.

The communication cost of both (DP-)Fed-mv-PPCA can be derived as well from the communication cost of distributed PPCA \citep{elgamal2015spca}, by considering that each center $c$ will communicate to the central server the parameter set: $\boldsymbol{\theta}_c:=\{\boldsymbol{\mu}_c^{(k)},W_c^{(k)},{\sigma_c^{(k)}}^2\}_{k\in K_c}$. For every $k$, the communication cost of $\theta_c^{(k)}$ is $\mathcal{O}(d_kq)$. Consequently, the global communication cost of $\boldsymbol{\theta}_c$ will be $\mathcal{O}(q\sum_{k\in K_c}d_k):=\mathcal{O}(d_cq)$, which is the same communication cost of standard PPCA for a $d_c$-dimensional dataset.

\section{Applications}\label{Sec4}

\subsection{Materials}

In the preparation of this article we used two datasets. 

\textbf{Synthetic dataset (SD):} using the generative model described in Section \ref{sec:meth}, we generated 400 observations consisting of $k=3$ views of dimension $d_1=15, d_2=8, d_3=10$ respectively. Each view was generated from a common 5-dimensional latent space. We randomly chose parameters $W^{(k)}, \boldsymbol{\mu}^{(k)}, {\sigma^{(k)}}$. Finally, to simulate heterogeneity, a randomly chosen sub-sample composed by 250 observations was shifted in the latent space by a randomly generated vector: this allowed to simulate the existence of two distinct groups in the population.

\textbf{Alzheimer's Disease Neuroimaging Initiative  dataset (ADNI)\footnote{The ADNI project was launched in 2003 as a public-private partnership, led by Principal Investigator Michael W. Weiner, MD. The primary goal of ADNI was to test whether serial magnetic resonance imaging (MRI), positron emission tomography (PET), other biological markers, and clinical and neuropsychological assessments can be combined to measure the progression of early Alzheimer's disease (AD) (see \url{www.adni-info.org} for up-to-date information).}:} 
we consider 311 participants extracted from the ADNI dataset, among cognitively normal (NL) (104 subjects) and patients diagnosed with AD (207 subjects). All participants are associated with multiple data views: cognitive scores including MMSE, CDR-SB, ADAS-Cog-11 and RAVLT (CLINIC), Magnetic resonance imaging (MRI), Fluorodeoxyglucose-PET (FDG) and AV45-Amyloid PET (AV45) images. 
MRI morphometrical biomarkers were obtained as regional volumes using the cross-sectional pipeline of FreeSurfer v6.0 and the Desikan-Killiany parcellation~\citep{fischl2012freesurfer}. Measurements from AV45-PET and FDG-PET were estimated by co-registering each modality to their respective MRI space, normalizing by the cerebellum uptake and by computing regional amyloid load and glucose hypometabolism using PetSurfer pipeline~\citep{greve2014cortical} and the same parcellation. Features were corrected beforehand with respect to intra-cranial volume, sex and age using a multivariate linear model. Data dimensions for each view are: $d_{\textrm{CLINIC}}=7$, $d_{\textrm{MRI}}=41$, $d_{\textrm{FDG}}=41$ and $d_{\textrm{AV45}}=41$. Further details on the demographics of the ADNI sample are provided in Appendix B, Table \ref{Tab_data1}.

\subsection{Benchmark}

We compare our method to two state-of-the art data assimilation methods: Variational Autoencoder (VAE)~\citep{kingma2014stochastic} and multi-channel VAE (mc-VAE)~\citep{antelmi2019sparse}. To maintain the modeling setup consistent across methods, both auto-encoders were tested by considering linear encoding and decoding mappings. In order to obtain the federated version of VAE and mc-VAE we use FedAvg~\citep{mcmahan2017communication}, which is specifically conceived for stochastic gradient descent optimization. Additional tests were performed by considering non-linear VAEs (2-layers for both encoding and decoding architectures), and FedProx as additional regularized FL aggregation method (results in Supp. Table \ref{ADNI_fedprox} and Supp. Figure \ref{fig_adni_fedprox}). For all optimization methods and federation schemes we set to 100 the total number of communication rounds, of 15 epochs each, with the default learning rate ($10^{-3}$).

% Moreover, we also tested the FedProx aggregation scheme~\citep{li2018federated}, with the proximal term $\lambda$ varying from 0.01 to 0.5: this method is supposed to improve convergence in case of heterogeneous data distribution. Nevertheless we do not observe a significant improvement with FedProx: results concerning this alternative federation scheme are provided in supplementary Table \ref{ADNI_fedprox} and supplementary Figures \ref{fig_adni_fedprox}-\ref{fig_sd_fedprox}.

\subsection{Results}\label{ssec:4.3}

We apply Fed-mv-PPCA to both SD and ADNI datasets, and quantify the quality of reconstruction and identification of the latent space with respect to the increasing number of centers, $C$, and the increasing data heterogeneity. We investigate also the ability of Fed-mv-PPCA in estimating the data variability and predicting the distribution of missing views. To this end, we consider 4 different scenarios of data distribution across multiple centers, detailed in Table~\ref{tab:scenarios}.

\begin{table}[ht]
\caption{Distribution of Datasets Across Centers.}
\label{tab:scenarios}
\centering
\begin{tabular}{lp{11cm}}
\textbf{Scenario} & \textbf{Description}   \\ 
\hline 
IID & Data are iid distributed across $C$ centers with respect to groups and for all subjects a complete data raw  is provided\\
G & Data are non-iid distributed with respect to groups across $C$ centers: $C/3$ centers includes subjects from both groups; $C/3$ centers only subjects from group 1 (AD in the ADNI case); $C/3$ centers only subjects from group 2 (NL for ADNI). All views have been measured in each center. \\
K & $C/3$ centers contribute with observations for all views; in $C/3$ centers the second view (MRI for ADNI) is missing; in $C/3$ centers the third view (FDG for ADNI) is missing. Data are iid distributed across $C$ centers with respect to groups. \\
G/K & Data are non-iid distributed (scenario G) and there are missing views (scenario K).
\end{tabular}
\end{table}

For each experiment considered hereafter with Fed-mv-PPCA, we perform 3-fold Cross Validation (3CV) tests. For every test, local parameters are initialized randomly (\emph{i.e.} no prior is provided by the master at the beginning), and the number of rounds is set to 100. Each round consists of 15 iterations for local MAP optimization, except the initialization round, which consists of 30 EM iterations. Finally, when a centralized setting is tested, the number of rounds is set to 1 and the number of EM iterations to 800.

\subsubsection{Model selection}

The latent space dimension $q$ is an user defined parameter, with the only constraint $q< \min_k\{d_k\}$. To assess the optimal $q$, 
we consider the IID scenario and let 
$q$ vary. We perform 10 times a 3-fold Cross Validation (3-CV), and split the train dataset across 3 centers.
The resulting models are compared using the WAIC criterion~\citep{gelman2014understanding}. In addition, we consider the Mean Absolute reconstruction Error (MAE) in an hold-out test dataset: the MAE is obtained by evaluating the mean absolute distance between real data and data reconstructed using the global distribution. 
Figure~\ref{Fig_WAIC_cen} shows the evolution of WAIC and MAE with respect to the latent space dimension. 

\begin{figure}[ht]
     \begin{subfigure}[b]{0.4\textwidth}
         \centerline{\includegraphics[width=\textwidth]{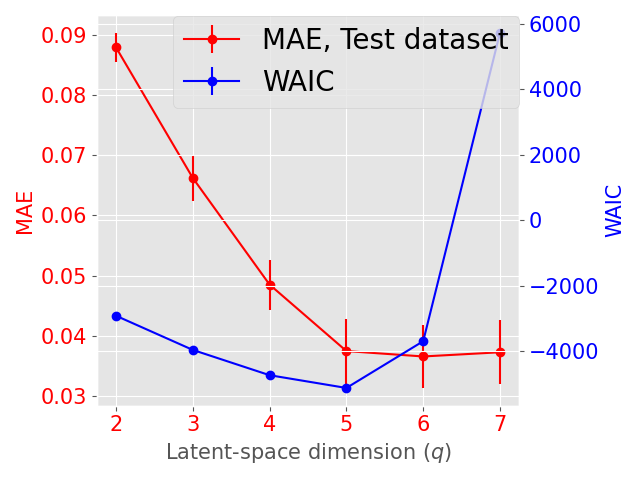}}
         \caption{SD}
     \end{subfigure}
     \hfill
     \begin{subfigure}[b]{0.4\textwidth}
         \centerline{\includegraphics[width=\textwidth]{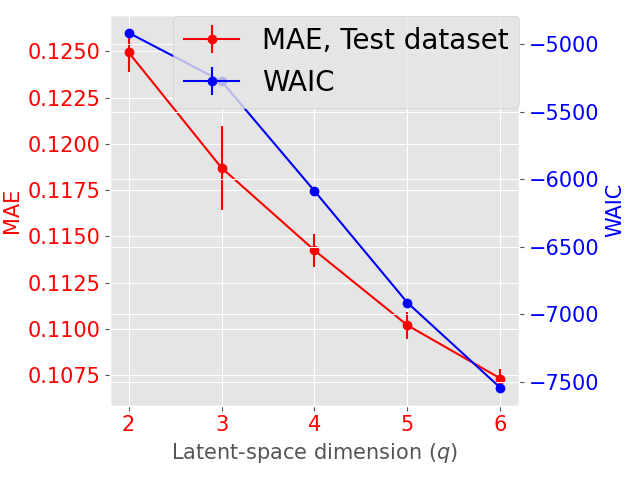}}
         \caption{ADNI}
     \end{subfigure}
\caption{WAIC score and MAE for (a) the SD dataset and (b) the ADNI dataset. In both figures, the left y-axis scaling describes the MAE while the right y-axis scaling corresponds to the WAIC score. }\label{Fig_WAIC_cen}
\end{figure}

\begin{figure}[ht]
\centering
\includegraphics[scale=0.4]{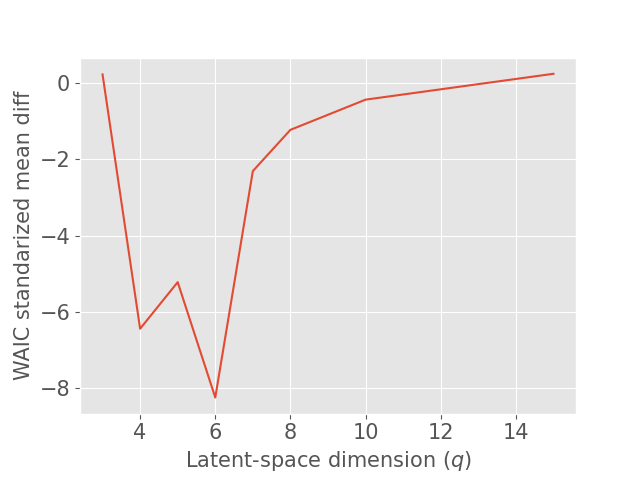}
\caption{Standardized mean differences between the WAIC score for ADNI data, computed as $(\textrm{mean}(\textrm{WAIC}_q)-\textrm{mean}(\textrm{WAIC}_{q-1}))/\sqrt{\textrm{var}(\textrm{WAIC}_q)/N_q-\textrm{var}(\textrm{WAIC}_{q-1})/N_{q-1}}$, $2<q\leq15$.   }\label{WAIC_std_diff}
\end{figure}

Concerning the SD dataset, the WAIC suggests $q=5$ latent dimensions (Figure \ref{Fig_WAIC_cen} (a)), hence demonstrating the ability of Fed-mv-PPCA to correctly recover the ground truth latent space dimension used to generate the data. Analogously, the MAE improves drastically up to the dimension $q=5$, and subsequently stabilizes. For ADNI, the MAE improves for increasing latent space dimensions, and we obtain the best WAIC score for $q=6$. 
%, suggesting that a high-capacity model is preferable to describe this larger dimensional dataset.
In this case, one can notice that both the WAIC and MAE keep decreasing when considering $q$ varying from 1 to 6. 
%, a condition imposed to insure that the latent dimension remains smaller than each view-specific dimension. Indeed, further increasing  $q$ above the original dimension can entail overfitting, while increasing model complexity and communication costs (see Section \ref{subsec:complexity}). Furthermore, i
In Figure \ref{WAIC_std_diff} we display the standardized mean differences of WAIC scores for $2<q\leq15$: increasing the latent dimension $q$ above 6 implies a mild relative improvement of the WAIC, while requiring a computationally more complex model and higher communication costs (see Section \ref{subsec:complexity}). This ultimately indicates that the choice of $q=6$ is a reasonable compromise for the ADNI database, allowing to efficiently capture most data variability, while remaining coherent with the model hypotheses (\emph{cf} $q<\min_k{d_k}$). Additionally, we should stress that when the $k$-th view dimension is smaller or equal to the latent dimension ($d_k\leq q$ for some $k$), we assumed that only the first $d_k-1$ columns of $W_c^{(k)}$ were effectively contributing for the latent projection of view $k$, and we forced the remaining columns of $W_c^{(k)}$ to be filled of zeros. For completeness, Supplementary Figure \ref{fig_waic_q_15} provides the evolution of WAIC for $q>6$ and shows that a latent dimension choice above $q=6$ is associated to a generally higher variance, suggesting less stable models and results.

It is worth noting that despite the agreement of MAE and WAIC for both datasets, the WAIC has the competitive advantage of providing a natural and automatic model selection measure in Bayesian models, which does not require testing data, conversely to MAE.

In the following experiments, we set the latent space dimension $q=5$ for the SD dataset and $q=6$ for the ADNI dataset. 

\subsubsection{Increasing heterogeneity across datasets}\label{sssec:heterogeneity}
% \todo[inline]{Initial phase with 30 epochs: results should be updated}
 
\begin{table}[!t]
\caption{Results on ADNI dataset for all scenarios, and comparison with VAE and mc-VAE.} \label{ADNI_2}
\centering
\begin{tabular}{l|c|l|l|l|l}
\textbf{Scenario} & \textbf{Centers} & \textbf{Method} & \textbf{MAE Train} & \textbf{MAE Test} & \textbf{Accuracy in LS}\\
\hline %& & & & \\
% IID, C=1
\multirow{ 12}{*}{IID} & \multirow{3}{*}{\makecell{1 \\ (centralized \\ case)}} & \textbf{Fed-mv-PPCA} & \textbf{0.0805$\boldsymbol\pm$0.0003} & \textbf{0.1110$\boldsymbol\pm$0.0011} & 0.8680$\pm$0.0379\\
% & & \textbf{PPCA} & \textbf{0.0801$\boldsymbol\pm$0.0004} & \textbf{0.1111$\boldsymbol\pm$0.0010} & \textbf{0.8841$\pm$0.0214}\\
& & VAE & 0.1055$\pm$0.0017 & 0.1344$\pm$0.0019 & 0.8003$\pm$0.0409\\
& & mc-VAE & 0.1382$\pm$0.0009 & 0.1669$\pm$0.0020 & \textbf{0.8727$\boldsymbol\pm$0.0319}\\\cline{2-6}
% IID, C=3
 & \multirow{4}{*}{3} & \textbf{Fed-mv-PPCA} & \textbf{0.1027$\boldsymbol\pm$0.0015} & \textbf{0.1073$\boldsymbol\pm$0.0004} & 0.8652$\pm$0.0270\\
& & DP-Fed-mv-PPCA & 0.1304$\pm$0.0047 & 0.1304$\pm$0.0041 & 0.8321$\pm$0.0388\\
& & VAE & 0.1172$\pm$0.0022 & 0.1192$\pm$0.0015 & 0.8289$\pm$0.0383
\\
& & mc-VAE & 0.1602$\pm$0.0035 & 0.1567$\pm$0.0017 & \textbf{0.8850$\boldsymbol\pm$0.0262}\\\cline{2-6} 
% IID, C=6
& \multirow{4}{*}{6} & \textbf{Fed-mv-PPCA} & \textbf{0.1203$\boldsymbol\pm$0.0042} & \textbf{0.1074$\boldsymbol\pm$0.0007} & 0.8742$\pm$0.0267\\
& & DP-Fed-mv-PPCA & 0.1489$\pm$0.0051 & 0.1295$\pm$0.0029 & 0.8502$\pm$0.0347\\
& & VAE & 0.1357$\pm$0.0042 & 0.1191$\pm$0.0014 & 0.8224$\pm$0.0377\\
& & mc-VAE & 0.1840$\pm$0.0054 & 0.1563$\pm$0.0017 & \textbf{0.8894$\boldsymbol\pm$0.0230}\\
\hline
% G, C=3
\multirow{ 6}{*}{G} & \multirow{3}{*}{3} & \textbf{Fed-mv-PPCA} & \textbf{0.1077$\boldsymbol\pm$0.0090} & \textbf{0.1096$\boldsymbol\pm$0.0011} & \textbf{0.8409$\pm$0.0293}\\
& & DP-Fed-mv-PPCA & 0.1362$\pm$0.0117 & 0.1340$\pm$0.0067 & 0.7977$\pm$0.0480\\
& & VAE & 0.1212$\pm$0.0077 & 0.1219$\pm$0.0015 & 0.7962$\pm$0.0440\\
% & & {\color{red}VAE+FedProx} & {\color{red} $\pm$} & {\color{red} $\pm$} & {\color{red} $\pm$}\\
& & mc-VAE & 0.1677$\pm$0.0156 & 0.1611$\pm$0.0025 & 0.8210$\pm$0.0464\\\cline{2-6}
% & & {\color{red}mc-VAE+FedProx} & {\color{red} $\pm$} & {\color{red} $\pm$} & {\color{red} $\pm$}\\\cline{2-6}
% G, C=6
& \multirow{ 3}{*}{6} & \textbf{Fed-mv-PPCA} & \textbf{0.1264$\boldsymbol\pm$0.0126} & \textbf{0.10912$\boldsymbol\pm$0.0011} & \textbf{0.8168$\boldsymbol\pm$0.0324}\\
& & DP-Fed-mv-PPCA & 0.1585$\pm$0.0158 & 0.1340$\pm$0.0065 & 0.7898$\pm$0.0407\\
& & VAE & 0.1401$\pm$0.0114 & 0.1202$\pm$0.0016 & 0.7882$\pm$0.0534\\
% & & {\color{red}VAE+FedProx} & {\color{red} $\pm$} & {\color{red} $\pm$} & {\color{red} $\pm$}\\
& & mc-VAE & 0.1924$\pm$0.0219 & 0.1589$\pm$0.0018 & 0.8085$\pm$0.0464\\
% & & {\color{red}mc-VAE+FedProx} & {\color{red} $\pm$} & {\color{red} $\pm$} & {\color{red} $\pm$}\\
\hline\hline %& & & & \\
% K, C=3
\multirow{ 4}{*}{K} & \multirow{2}{*}{3} & Fed-mv-PPCA & 0.0951$\pm$0.0086 & 0.1212$\pm$0.0109 & 0.8624$\pm$0.0303\\
& &  DP-Fed-mv-PPCA & 0.1208$\pm$0.0081 & 0.1462$\pm$0.0092 & 0.8357$\pm$0.0329\\\cline{2-6}
% K, C=6
& \multirow{2}{*}{6} & Fed-mv-PPCA & 0.1107$\pm$0.0106 & 0.1293$\pm$0.0162 & 0.8720$\pm$0.0308\\
& &  DP-Fed-mv-PPCA & 0.1434$\pm$0.0099 & 0.1604$\pm$0.0164 & 0.8515$\pm$0.0375\\
\hline %& & & & \\
% GK, C=3
\multirow{ 4}{*}{G/K} & \multirow{2}{*}{3} & Fed-mv-PPCA & 0.0995$\pm$0.0029 & 0.1271$\pm$0.0087 & 0.7338$\pm$0.0308\\
& & DP-Fed-mv-PPCA & 0.1287$\pm$0.0081 & 0.1547$\pm$0.0125 & 0.7164$\pm$0.0474\\\cline{2-6}
% GK, C=6
& \multirow{2}{*}{6} & Fed-mv-PPCA & 0.1173$\pm$0.0061 & 0.1268$\pm$0.0088 & 0.7469$\pm$0.0202\\
& & DP-Fed-mv-PPCA & 0.1463$\pm$0.0088 & 0.1523$\pm$0.0104 & 0.7174$\pm$0.0387
\end{tabular}
\end{table}
To test the robustness of Fed-mv-PPCA's results, 
for each scenario of Table~\ref{tab:scenarios}, we perform 10 times 3-CV to obtain train and test datasets, hence we split the train dataset across $C$ centers. 
We compare our method to VAE and mc-VAE, using the same partition of train and test datasets for CV.  For all methods we consider the MAE in both the train and test datasets, as well as the accuracy score in the Latent Space (LS) discriminating the groups (synthetically defined in SD or corresponding to the clinical diagnosis in ADNI). The classification was performed via Linear Discriminant Analysis (LDA) on the individual projection of test data in the latent space. 

In what follows we present a detailed description of results corresponding to the ADNI dataset. Results for the SD dataset are in line with what we observe for ADNI (see Supplementary Table \ref{SD_2} in Appendix B), and confirm that our method outperforms both VAE and mc-VAE in reconstruction in all scenarios. In addition, Fed-mv-PPCA outperforms in discrimination both methods in the non-iid setting, while mc-VAE shows slightly improved discriminating ability in the IID scenario.

Moreover, for the sake of completeness, supplementary Table \ref{ADNI_fedprox} and supplementary Figure \ref{fig_adni_fedprox} provide results for both VAE and mc-VAE with two layers, as well as both methods with one layer and using FedProx as robust aggregation scheme with the proximal term $\lambda$ varying from 0.01 to 0.5: this method aims at improving convergence in case of heterogeneous data distributions. No significant improvement as been observed comparing to the FedAvg scheme for the considered datasets and settings, while non linear models are associated with a negligible improvement in testing compared to the linear variational autoencoders.

\paragraph{IID distribution.}

We consider the IID scenario and split the train dataset across 1 to 6 centers. 
Table \ref{ADNI_2} shows that results from Fed-mv-PPCA are stable when moving from a centralized to a federated setting, and when considering an increasing number of centers $C$. 
 We only observe a degradation of the MAE in the train dataset, but this does not affect the performance of Fed-mv-PPCA in reconstructing the test data. Moreover, irrespectively from the number of training centers, Fed-mv-PPCA outperforms VAE and mc-VAE in reconstruction. 

\paragraph{Heterogeneous distribution.}

\begin{figure}[ht]
\centering
     \begin{subfigure}[b]{1.1\textwidth}
         \centerline{\includegraphics[width=1.1\textwidth]{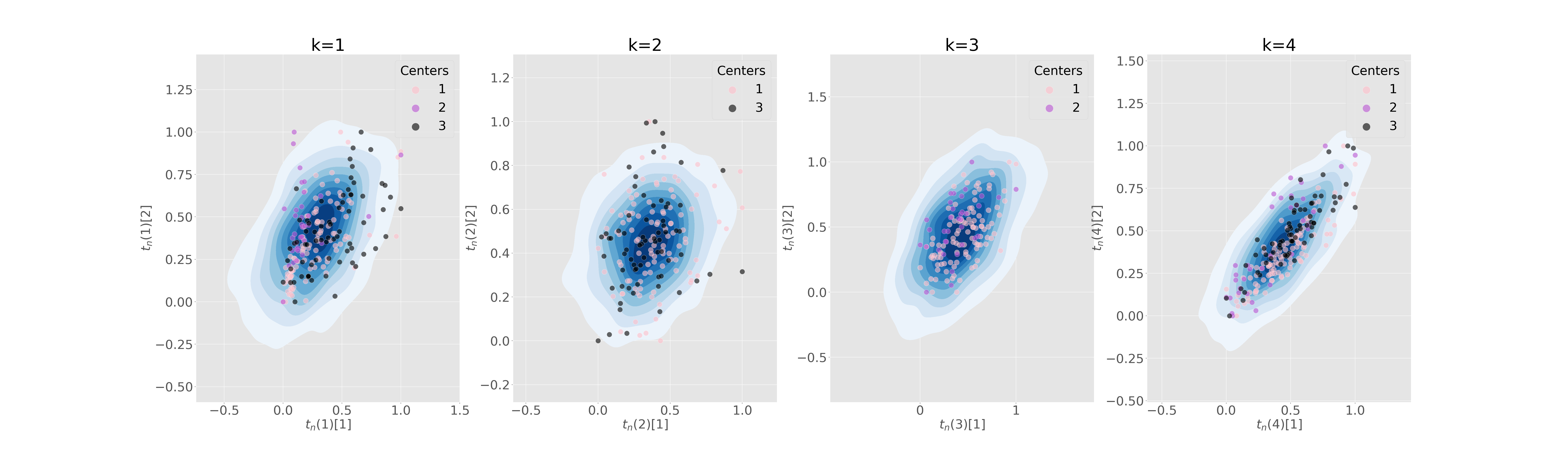}}
         \caption{Original space}
     \end{subfigure}
     \hfill
     \begin{subfigure}[b]{0.4\textwidth}
         \centerline{\includegraphics[width=.86\textwidth]{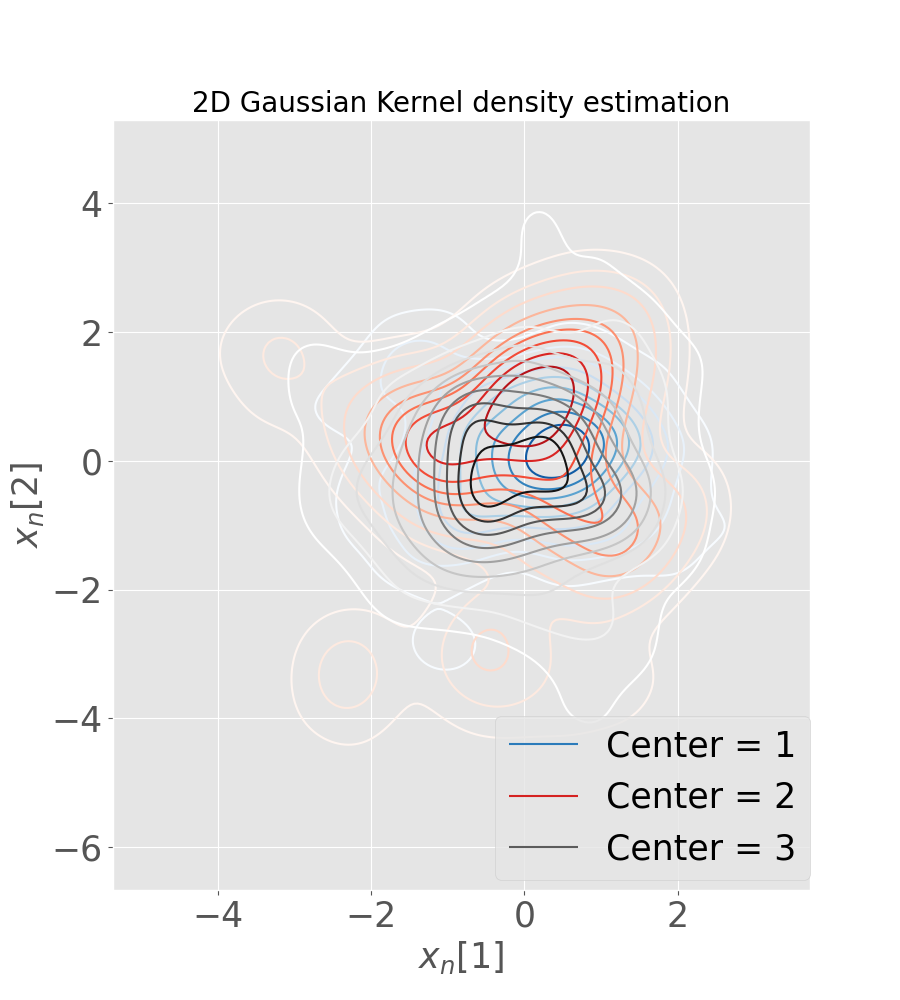}}
         \caption{Latent space}
    \end{subfigure}
     \hfill
     \begin{subfigure}[b]{0.5\textwidth}
         \includegraphics[width=1\linewidth]{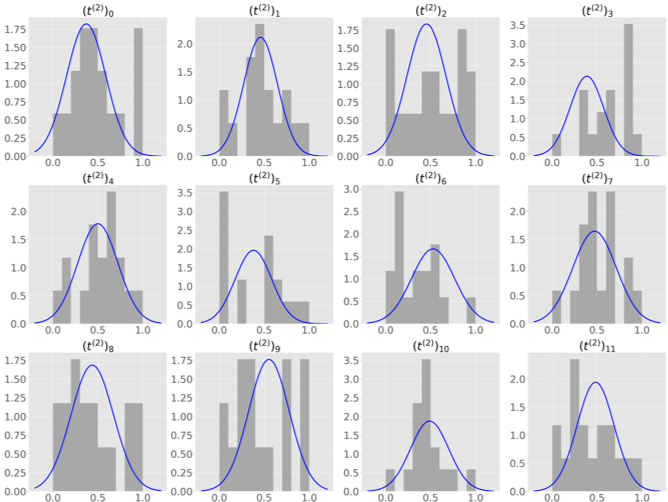}
         \caption{Missing views imputation}
     \end{subfigure}
\caption{G/K scenario. First two dimensions for (a) sampling from posterior distribution of latent variables $\mathbf{x}_{c,n}$, and (b) predicted distribution $\mathbf{t}_{c,n}^{(k)}$ against real data. (c) Predicted testing distribution (blue curve) of sample features of the missing MRI view against real data (histogram). }\label{Variability}
\end{figure}

We simulate an increasing degree of heterogeneity in 3 to 6 local datasets, to further challenge the models in properly recovering the global data. 
In particular, we consider both a non-iid distribution of subjects across centers, and missing views in some local dataset. 
It is worth noting that scenarios implying datasets with missing views cannot be handled by VAE nor by mc-VAE, hence in these cases we reported only results obtained with our method. 

In Table~\ref{ADNI_2} we report the average MAEs and Accuracy in the latent space for each scenario, obtained over 10 tests for the ADNI dataset. Fed-mv-PPCA is robust despite an increasing degree of heterogeneity in the local datasests.  We observe a slight deterioration of the MAE in the test dataset in the more challenging non-iid cases (scenarios K and G/K), while we note a drop of the classification accuracy in the most heterogeneous setup (G/K). Nevertheless, Fed-mv-PPCA demostrates to be more stable and to perform better than VAE and mc-VAE when statistical heterogeneity is introduced.

Figure~\ref{Variability} (a) shows the sampling posterior distribution of the latent variables, while in Figure~\ref{Variability} (b) we plot the predicted global distribution 
%of the corresponding original space
 against observations, for the G/K scenario and considering 3 training centers. We notice that the variability  of centers is well captured, in spite of the heterogeneity of the distribution in the latent space. In particular center 2 and center 3 have two clearly distinct means: this is due to the fact that subjects in these centers belong to two distinct groups (AD in center 2 and NL in center 3). 
Despite this, Fed-mv-PPCA is able to reconstruct 
correctly all views, even if 2 views are completely missing in some local datasets (MRI is missing in center 2 and FDG in center 3).

After convergence of Fed-mv-PPCA, each center is supplied with global distributions for each parameter: data corresponding to each view can therefore be simulated, 
even if some are missing in the local dataset. Considering the same si\-mu\-la\-tion in the challenging G/K scenario, 
%as in Figure~\ref{Variability},
in Figure~\ref{Variability} (c) we plot the global distribution of some randomly selected features of a missing imaging view in the test center, against ground truth density histogram, from the original data. The global distribution provides an accurate description of the missing MRI view. Supplementary Figure \ref{Prediction_ADNI_supp} shows imputation for all features of the missing MRI and FDG views.

\subsubsection{Differentially private Fed-mv-PPCA}\label{sec:dp_res}

% Previous results were obtained using Fed-mv-PPCA without any formal privacy guarantee. 
We repeated all experiments described in Section \ref{sssec:heterogeneity}, using Fed-mv-PPCA with differential privacy. 
% : the corresponding results are reported in Table \ref{ADNI_2} (DP-Fed-mv-PPCA rows). 
For each parameter $\theta_c\in\boldsymbol{\theta}_c$ we set $\varepsilon=10$, $\delta=0.01$, except when stated otherwise. Finally, to perform difference clipping (Algorithm \ref{algo2}), we set the maximal $l_p$ norm of the difference between the updated parameter at round $r$ and the prior, $\overline{\Delta\boldsymbol{\theta}}_c[r]$, to be $\sigma_{\widetilde{\boldsymbol{\theta}}}[r-1]$.

\paragraph{DP parameters utility.}

We tested the utility of global parameters obtained with the differentially private Algorithm \ref{algo2}, to appreciate if data reconstruction and accuracy in the latent space are well preserved when the perturbation is performed at the client level (see Table \ref{ADNI_2}, DP-Fed-mv-PPCA rows). 
As expected, we observe a deterioration of previous results, which increases with the number of training centers, due to the communication of a larger number of perturbed parameters. Nevertheless, results remain still coherent, and illustrate the utility of the differentially private global parameters. 
%In particular Figure \ref{eps_clip} shows the relationship between  global parameters utility and both $\varepsilon$ and the multiplicative constant used for difference clipping. We note that when $(\varepsilon,\delta)$ are fixed to $(1,0.01)$, the clipping constant should be at most 0.2 to preserve a reasonable utility of the model outputs. This further stresses the need of carefully tuning these DP parameters to ensure a good balance between privacy and utility.
In particular Figure \ref{eps_clip} shows how $\varepsilon$ and the multiplicative constant used for difference clipping affect the ability of the optimzed DP global parameters in preserving a meaningful separation of subjects by diagnosis in the test set. For instance, we note that when $(\varepsilon,\delta)$ are fixed to $(1,0.01)$, the clipping constant should be at most 0.2 to preserve a reasonable utility of the model outputs, in comparison to the one obtained using Fed-mv-PPCA (reported in Figure \ref{eps_clip}, \emph{not DP} column). This further stresses the need of carefully tuning these DP parameters to ensure a good balance between privacy and utility.

\begin{figure}[ht]
     \begin{subfigure}[b]{0.4\textwidth}
         \centerline{\includegraphics[width=\textwidth]{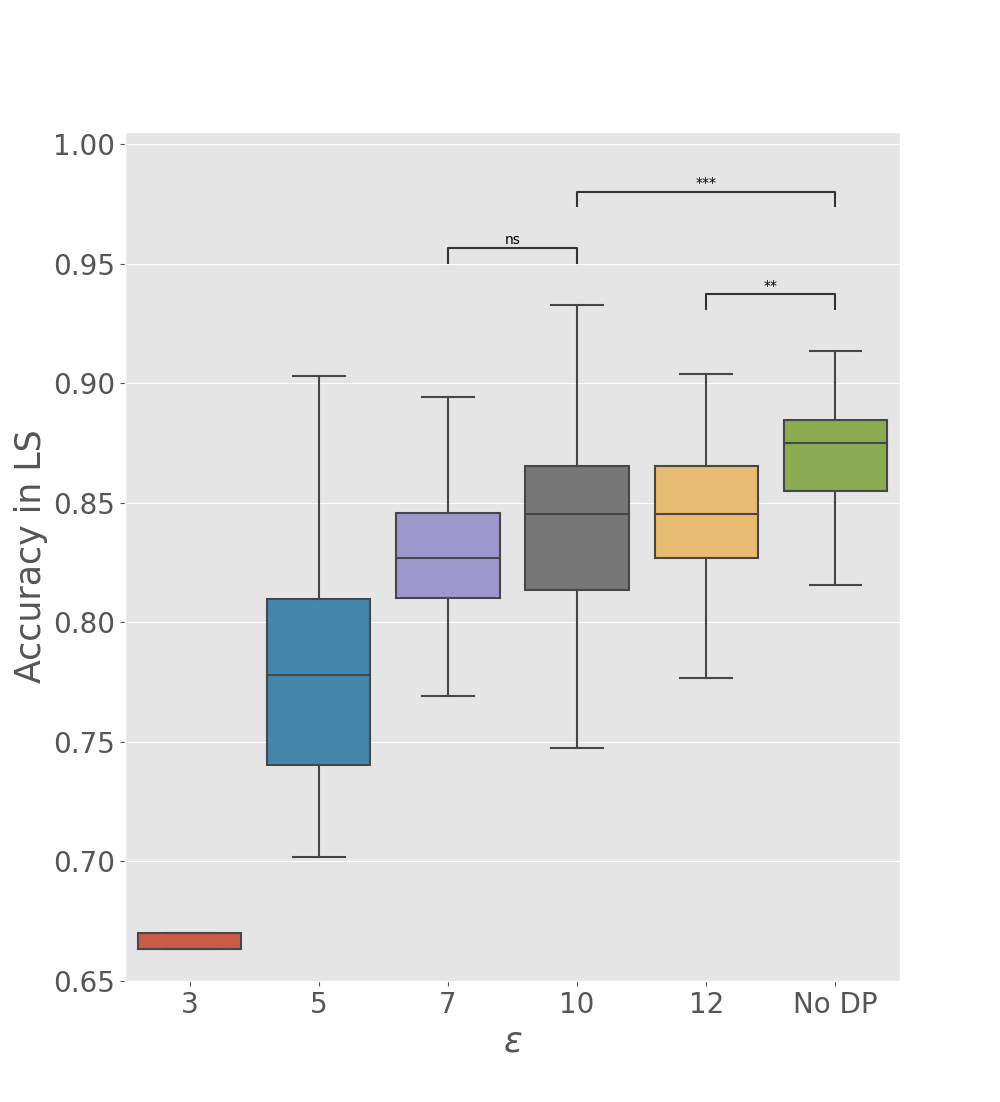}}
         \caption{$\delta=0.01$, clipping constant$=1$}
     \end{subfigure}
     \hfill
     \begin{subfigure}[b]{0.4\textwidth}
         \centerline{\includegraphics[width=\textwidth]{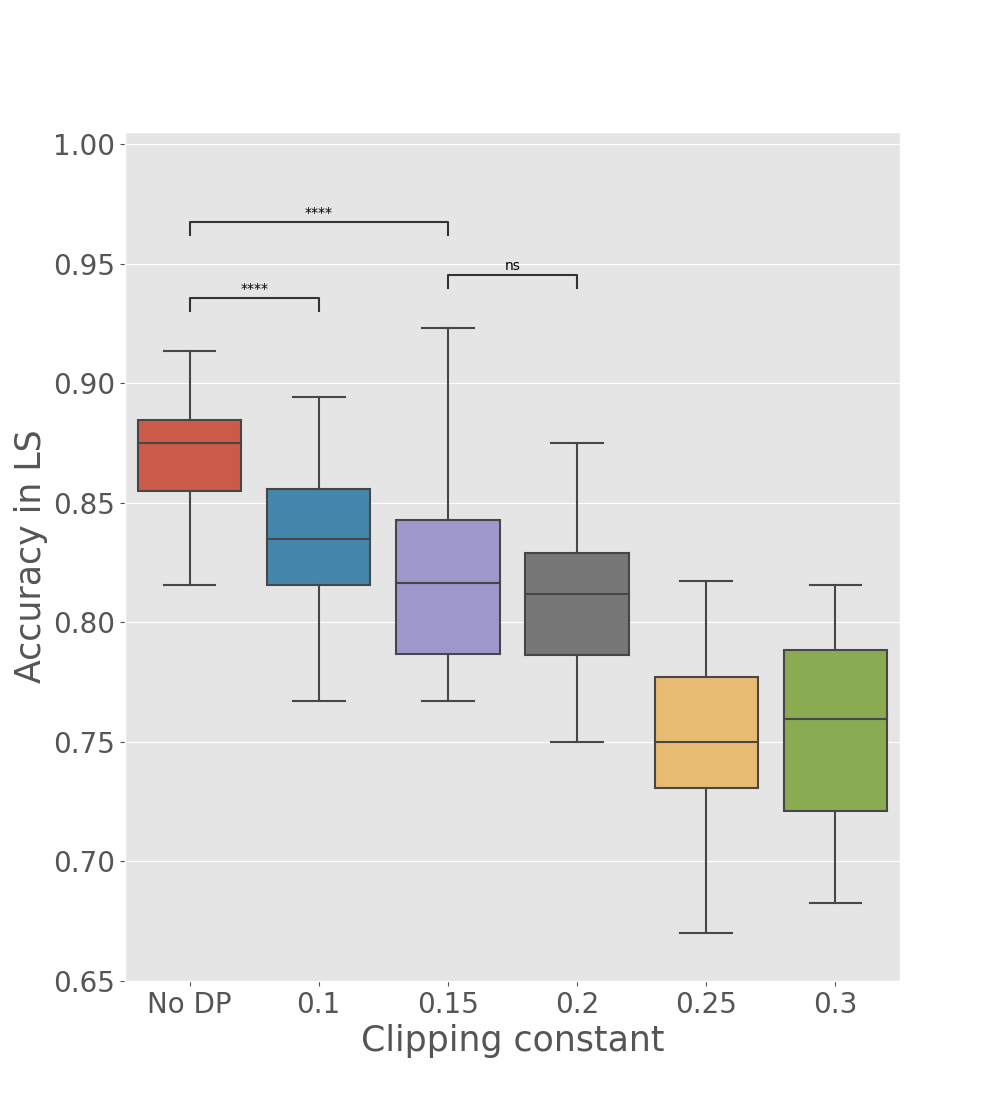}}
         \caption{$\varepsilon=1, \delta=0.01$}
     \end{subfigure}
\caption{DP-Fed-mv-PPCA performance in preserving subjects separation in the Latent Space (LS) by diagnosis, varying (a) $\varepsilon$ and (b) the multiplicative constant for difference clipping. Results for Fed-mv-PPCA (without DP) are reported for comparison purposes. $\ast\ast:=p\leq1.e-2$, $\ast\ast\ast:=p\leq1.e-3$, $\ast\ast\ast\ast:=p\leq1.e-4$.
}\label{eps_clip}
\end{figure}

\paragraph{Evolution of the standard deviation of global parameters and convergence.}

To better understand the effect of performing difference clipping with respect to the priors, in Figure \ref{sigma_tilde_theta} (a-b) we plot the median evolution of the estimated standard deviation for each global parameter during training in the G/K scenario, comparing Fed-mv-PPCA and DP-Fed-mv-PPCA. When DP is not introduced, we can see that all global parameters' standard deviations converge, as expected, indicating harmonization of local parameters during training. In particular, it is worth noticing that the clinical view displays higher variability for both intercept and noise parameters. Indeed, the clinical view is the most discriminant one between healthy and Alzheimer patients, and results plotted in Figure \ref{sigma_tilde_theta} are obtained under a non-iid scenario. Furthermore, one can notice the low magnitude of the standard deviations for $\widetilde{\mu}^{(FDG)}$ and $\widetilde{\sigma}^{(MRI)}$: this may be explained by the fact that there are less centers contributing to the estimation of both FDG- and MRI-specific parameters, since in the G/K scenario the FDG and MRI views are missing in some centers. On the other hand, when differential privacy is introduced we tend to loose information concerning variability of global parameters: in this case all standard deviations drop towards 0 after approximately 20 communication rounds, meaning that the final global parameters distributions are strongly concentrated around their mean. 

\begin{figure}[ht]
\centering
     \begin{subfigure}[b]{0.55\textwidth}
         \centerline{\includegraphics[width=\textwidth]{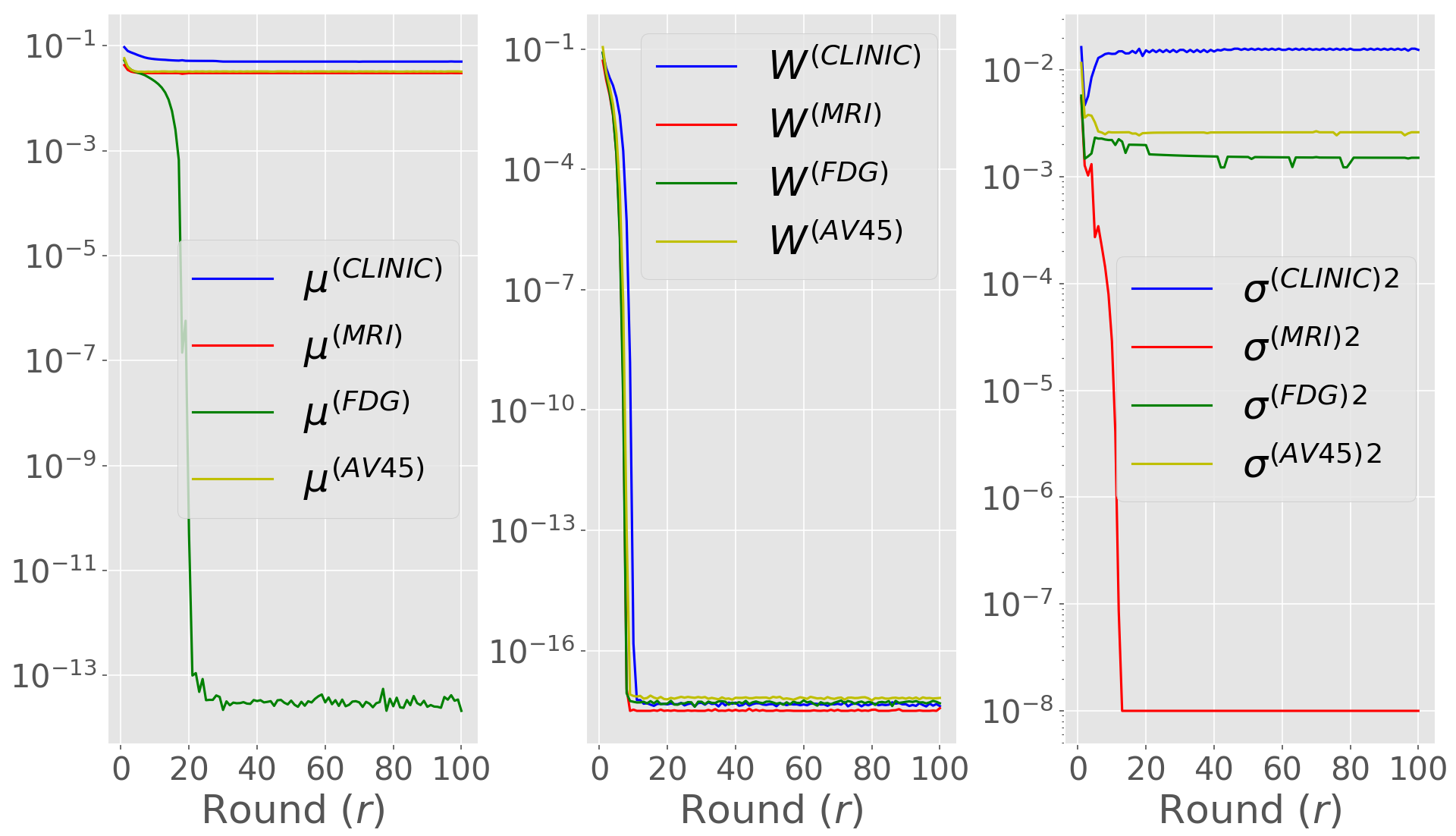}}
         \caption{Fed-mv-PPCA}
     \end{subfigure}
     \hfill
     \begin{subfigure}[b]{0.55\textwidth}
         \centerline{\includegraphics[width=\textwidth]{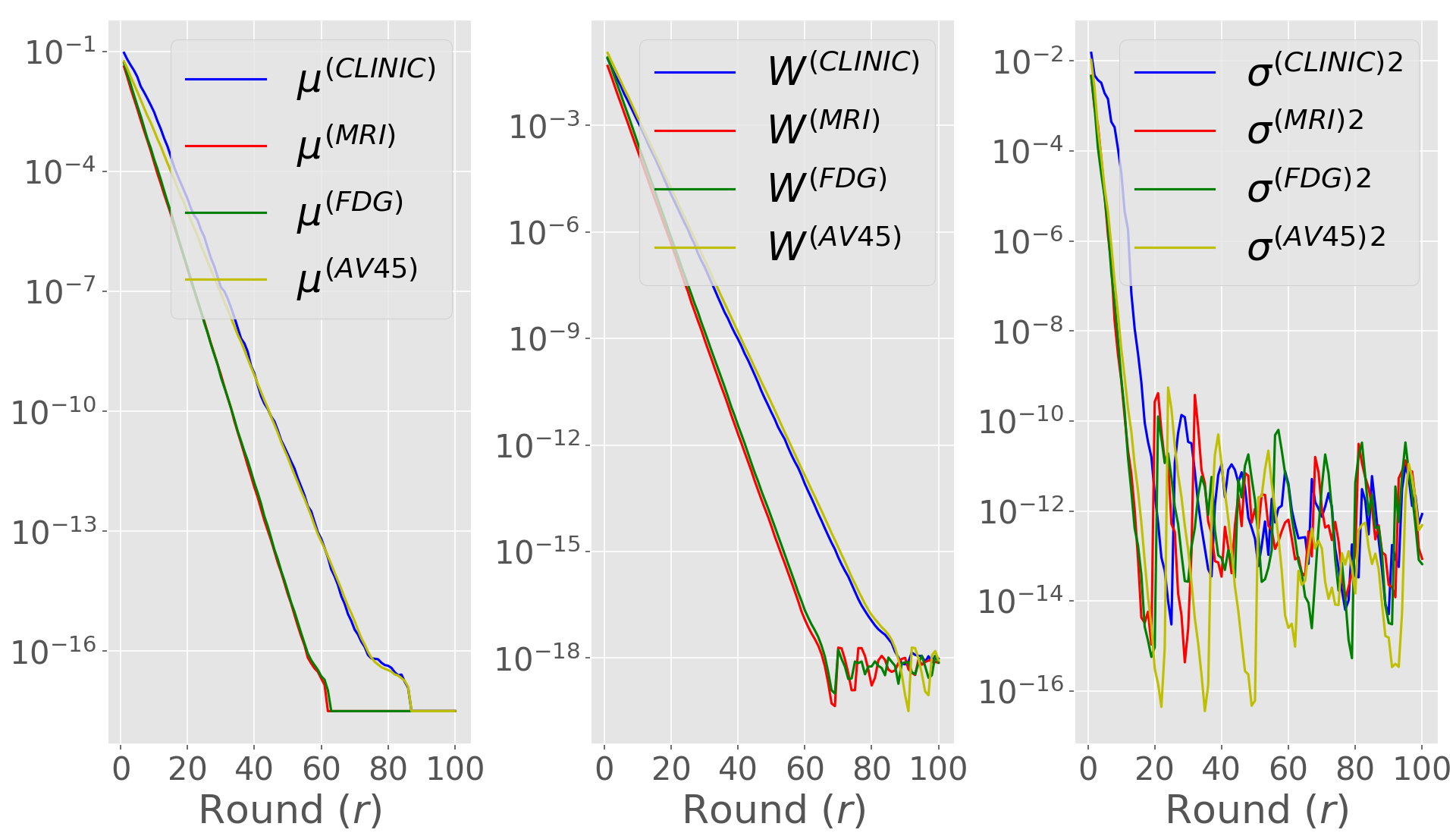}}
         \caption{DP-Fed-mv-PPCA}
     \end{subfigure}
     \hfill
     \begin{subfigure}[b]{0.4\textwidth}
         \centerline{\includegraphics[width=\textwidth]{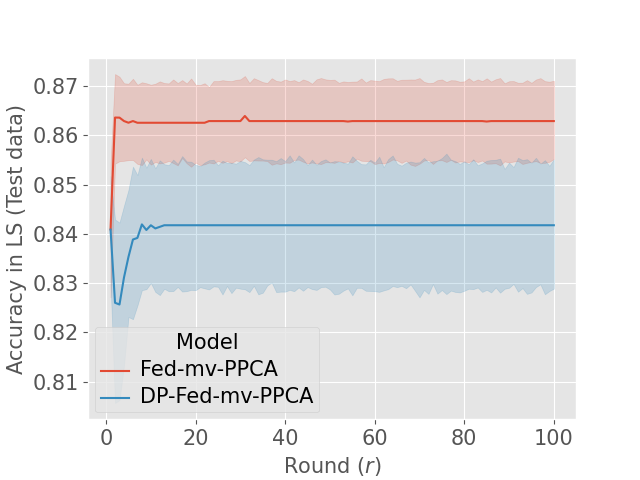}}
         \caption{Comparison of accuracy metric}
     \end{subfigure}
     %\;
     %\begin{subfigure}[b]{0.35\textwidth}
     %    \centerline{\includegraphics[width=\textwidth]{MELBA/Figures/LL_LLpert_GK_3.png}}
     %    \caption{Expected Log-likelihood}
     %\end{subfigure}
\caption{Evolution of the standard deviation (in $\log_{10}$ scale) of all global parameters for the GK scenario using 3 centers: comparison between (a) Fed-mv-PPCA and (b) DP-Fed-mv-PPCA. (c) Accuracy in the latent space across round (mean and std over 10 3-CV tests performed in the IID case), comparing Fed-mv-PPCA and DP-Fed-mv-PPCA.
%(c) Evolution of the expected log likelihood evaluated  at each round using either $\boldsymbol{\theta}_c$ (blue curve) or $\overline{\boldsymbol{\theta}}_c$ (red curve).
}\label{sigma_tilde_theta}
\end{figure}

Finally, we empirically investigate the convergence of DP-Fed-mv-PPCA. The convergence of the EM algorithm for PPCA has already been commented by \cite{tipping1999probabilistic}. Ne\-ver\-the\-less, in the case of DP-Fed-mv-PPCA, local parameters updates are performed using priors estimated at the master level from perturbed previous local updates. In addition, as commented above, the standard deviations of global parameters used as priors tend to decrease rapidly due to the clipping mechanism. Consequently, priors provided to the centers will be increasingly informative, affecting the algorithm convergence. Figure \ref{sigma_tilde_theta} (c) shows the mean evolution of the accuracy in the latent space (and for the test dataset) during successive rounds of both Fed-mv-PPCA and DP-Fed-mv-PPCA: mean and standard deviation are obtained by repeating 10 times a 3-CV test. 
%Despite the convergence of the algorithm seems to be reached in both cases, DP-Fed-mv-PPCA optimized parameters are clearly sub-optimal. 
Although the convergence of the algorithm seems to be reached in both cases, DP-Fed-mv-PPCA optimized parameters are clearly sub-optimal. 
Moreover, we notice a higher variability of the accuracy metric, as a consequence of the random perturbation performed over local parameters, which in turns affects the priors. Further insights are provided in Supp. Figure \ref{Convergence_global_supp}, showing the estimated global variance of the Gaussian noise, which is greater when using DP-Fed-mv-PPCA compared to Fed-mv-PPCA, indicating an estimated higher variability in the global dataset (\emph{i.e.} the ensemble of the local datasets). This is an expected consequence of the perturbation mechanism, which necessary affects the global model's performance.

\section{Conclusions 
%{\color{red} and Discussion}
}\label{Sec5}

In spite of the large amount of currently available multi-site biomedical data,  we still lack of reliable analysis methods to be applied in  multi-centric applications in compliance with privacy. To tackle this challenge, Fed-mv-PPCA proposes a hierarchical generative model to perform data assimilation of federated heterogeneous multi-views data. The Bayesian approach allows to naturally handle statistical heterogeneity across centers and missing views in local datasets, to provide an interpretable model of data variability and a valuable tool for missing data imputation. We show that  Fed-mv-PPCA can be further coupled with differential privacy. Compatibly with our Bayesian formulation, we provide formal privacy guarantees of the proposed federated learning scheme against potential private information leakage from the shared statistics. 

Our applications demonstrate that Fed-mv-PPCA is robust with respect to an increasing degree of heterogeneity across training centers, and provides high-quality data reconstruction, outperforming competitive methods in all scenarios. Moreover, when differential privacy is introduced, we provide an investigation of the method's performance according to different privacy budget scenarios.  It is worth noting that three DP hyperparameters play a key role, and could affect the performance of DP-Fed-mv-PPCA: the privacy budget parameters $(\epsilon, \delta)$, and the clipping constant multiplying $\sigma_{\widetilde{\boldsymbol{\theta}}}$. These parameters are tightly related and all contribute to determine the magnitude of the noise used for perturbing the updated difference $\overline{\Delta\boldsymbol{\theta}}$. Indeed, increasing either $\varepsilon$ or $\delta$, or reducing the multiplicative constant in the clipping mechanism, implies the addition of a smaller noise, hence the improvement of the overall utility of the global model. Nevertheless, smaller $\varepsilon$ and $\delta$ corresponds to higher  privacy guarantees.

Further extensions of this work are possible in several directions. The computational efficiency of Fed-mv-PPCA and its scalability to large datasets can be improved by leveraging on data sparsity and optimizing matrix multiplications and norm calculations as showed by \cite{elgamal2015spca}. In addition, introducing sparsity on the reconstruction weights is also expected to improve the robustness of the approach to non-informative dimensions and modalities. Another interesting research direction concerns the handling of missing data. Indeed, in this paper we considered Missing At Random (MAR) views in local datasets due to heterogeneous pipelines~\citep{rubin1976inference}.  
%Following the classification of missing data given by \cite{rubin1976inference}, we are talking here of Missing At Random (MAR) data: 
Fed-mv-PPCA could be extended to take into account and impute Missing Not At Random (MNAR) data as well, covering for instance the case of missing data due to self censoring, of interest in the biomedical context.

In this work we adopted DP to increase our framework's security, motivated by the need to derive explicit theoretical privacy guarantees for our model. 
Alternatively, some recent works propose to improve data privacy (and eventually model utility) in a federated setting by generating fake data through generative adversarial networks \citep{rajotte2021reducing,rasouli2020fedgan}. 
Despite 
%However, these approaches can not provide theoretical privacy guarantees, conversely to DP.
% We note however that 
 formal privacy guarantees cannot be provided by data augmentation methods, 
% and differential privacy techniques are not considered in the cited works.
%Contrarily to the aforementioned studies, our choice to adopt differential privacy is instead motivated by the need to derive explicit theoretical privacy guarantees for our model, which are currently not ensured by data augmentation methods. We agree that
% further quantification of privacy guarantees of generative approaches and
  their comparison to DP is a problem of great interest and should be further investigated.
%  , however we believe that this falls beyond the scope of the present study.

%From a theoretical viewpoint, a formal analysis of DP-Fed-mv-PPCA can be investigated, giving further insights on the effect of the choice of privacy parameters. 
In addition, we provided an experimental analysis of the convergence properties of DP in the proposed setting. In the future, formal convergence guarantees could be investigated, for example for the general optimization setting associating DP to EM. Furthermore, adaptive clipping strategies \citep{andrew2019differentially} could be investigated and employed to improve the convergence of DP-Fed-mv-PPCA and the final utility of global parameters. Finally, in order to improve the robustness of DP-Fed-mv-PPCA, non-Gaussian data likelihood and priors could be introduced in the future, to better account for heavy-tailed distributions defined by outliers datasets and centers.
% in case of highly perturbed parameters}.

%%%%%%%%%%%%%%%%%%%%%%%%%%%%%%%%%%%%%%%%%%%%%%%%%%%%%%%%%%%%%%%%%%%%%%%
% Mandatory Sections. Please complete, especially for final publication
%%%%%%%%%%%%%%%%%%%%%%%%%%%%%%%%%%%%%%%%%%%%%%%%%%%%%%%%%%%%%%%%%%%%%%%

% Acknowledgements.
% Please include any funding, intellectual contributions not included in the authorship, and any other acknowledgements.
\acks{This work received financial support by the French government, through the 3IA C\^ote d'Azur  Investments  in  the  Future  project  managed  by  the  National  Research Agency (ANR) with the reference number ANR-19-P3IA-0002, and by the ANR JCJC project Fed-BioMed, ref. num. 19-CE45-0006-01. The authors are grateful to the OPAL infrastructure from Universit\'e C\^ote d'Azur for providing resources and support.\\

Data   collection   and   sharing  for  this   project   was  funded   by  the   Alzheimer's   Disease Neuroimaging  Initiative  (ADNI)  (National  Institutes  of  Health  Grant  U01  AG024904)  and DOD  ADNI  (Department  of  Defense  award  number  W81XWH-12-2-0012).  ADNI  is  funded by  the   National  Institute  on  Aging,  the  National  Institute  of  Biomedical  Imaging  and Bioengineering, and through generous contributions from the following: AbbVie, Alzheimer's Association;  Alzheimer's  Drug  Discovery  Foundation;  Araclon  Biotech;  BioClinica,  Inc.; Biogen;   Bristol-Myers   Squibb   Company;   CereSpir,   Inc.;   Cogstate;   Eisai   Inc.;   ElanPharmaceuticals,  Inc.;  Eli  Lilly  and  Company;  EuroImmun;  F.  Hoffmann-La  Roche  Ltd and its  affiliated  company  Genentech,  Inc.;  Fujirebio;  GE  Healthcare;  IXICO  Ltd.;  Janssen Alzheimer    Immunotherapy    Research    $\&$    Development,    LLC.;    Johnson    $\&$    Johnson Pharmaceutical  Research  $\&$  Development  LLC.;  Lumosity;  Lundbeck;  Merck  $\&$  Co.,  Inc.; Meso  Scale  Diagnostics,  LLC.;  NeuroRx  Research;  Neurotrack  Technologies;  Novartis Pharmaceuticals Corporation; Pfizer Inc.; Piramal Imaging; Servier; Takeda Pharmaceutical Company;  andTransition  Therapeutics.  The  Canadian  Institutes  of  Health  Research  is providing  funds  to  support  ADNI  clinical  sites  in  Canada.  Private  sector  contributions  are facilitated by the Foundation for the National Institutes of Health (\url{www.fnih.org}). The grantee organization is the Northern California Institute for Research and Education, and the study is coordinated by the Alzheimer's Therapeutic Research Institute at the University of Southern California.  ADNI  data  are  disseminated   by  the   Laboratory  for  NeuroImaging  at  the University of SouthernCalifornia.}

% Ethical Standards.
% Please edit with the appropriate ethics considerations for your work. Include any pertinent IRB information, etc.
%
% Please note that the submission requirements included:
% The work presented must follow appropriate ethical standards in conducting research and writing the manuscript, following all applicable laws and regulations regarding treatment of animals or human subjects.
\ethics{The work follows appropriate ethical standards in conducting research and writing the manuscript, following all applicable laws and regulations regarding treatment of animals or human subjects.}

% Conflict of Interest
% Declaration of possible conflicts of interest: Authors must disclose any financial, organisational, commercial or personal conflicts of interest that might bias their work.
% If no conflicts, please say "We declare we don't have conflicts of interest."
\coi{The authors declare that they have no conflict of interests.
% The conflicts of interest have not been entered yet.
}

\bibliography{biblio}

% needed in general before appendices.
\newpage
\appendix % optional

\section*{Appendix A. Theoretical derivation of Fed-mv-PPC method}\label{AppendixA}%{Supplementary Material}

% \subsection*{Theoretical derivation of Fed-mv-PPC method}\label{supp_theory}

\subsection*{Problem setting}

We consider $C$ centers, each center $c\in\{1,\dots,C\}$ providing data from $N_{c}$ subjects, each consisting of $K_{c}\leq K$ views. Let $d_k$ be the dimension of data corresponding to the $k^{\textrm{th}}$-view, and $d:=\sum_{k=1}^K d_k$. \\

For each $k, c$ and each $n\in\{1,\dots,N_{c}\}$, the generative model is:

\begin{equation}\label{tcnkg_supp}
    \mathbf{t}_{c,n}^{(k)}=W_{c}^{(k)}\mathbf{x}_{c,n}+\boldsymbol{\mu}_{c}^{(k)}+\boldsymbol{\varepsilon}_{c}^{(k)},
\end{equation}
where:
\begin{itemize}
    \item $\mathbf{t}_{c,n}^{(k)}\in\mathbb{R}^{d_k}$ denotes the raw data of the $k^{\textrm{th}}$-view of the sample indexed by $n$ in center $c$, which belongs to group $g$.
    \item $\mathbf{x}_{c,n}\sim\mathcal{N}(0,\mathbb{I}_q)$ is a $q$-dimensional latent variable, $q\leq\min_k(d_k)$ being a suitable user-defined latent-space dimension.
    \item $W_{c}^{(k)}\in\mathbb{R}^{d_k\times q}$ provides the linear mapping between the two sets of variables for the $k^{\textrm{th}}$-view.
    \item $\boldsymbol{\mu}_{c}^{(k)}\in\mathbb{R}^{d_k}$ allows data corresponding to view $k$ to have a non-zero mean.
    % \item $\mathbf{a}_{c,g}\in\mathbb{R}^{q}$ determines a group-related shift in the latent space with respect to the reference group, for whom this parameter is set to $\boldsymbol{0}$ (here $a_1=\boldsymbol{0}$ since we supposed that the reference group is the first one).
    \item $\boldsymbol{\varepsilon}_{c}^{(k)}\sim\mathcal{N}\left(0,{\sigma_{c}^{(k)}}^2\mathbb{I}_{d_k}\right)$ is a Gaussian noise for the $k^{\textrm{th}}$-view.
\end{itemize}

A compact formulation for $\mathbf{t}_{c,n}$ (\emph{i.e.} considering all views concatenated) can be easily derived from Equation \eqref{tcnkg_supp}:
\begin{equation}\label{tcng_supp}
    \mathbf{t}_{c,n}=W_{c}\mathbf{x}_{c,n}+\boldsymbol{\mu}_{c}+\boldsymbol{\varepsilon}_{c},
\end{equation}
where:
\begin{itemize}
    \item $\mathbf{t}_{c,n}=\left[{\mathbf{t}_{c,n}^{(1)}}^T,\dots,{\mathbf{t}_{c,n}^{(K)}}^T\right]^T\in\mathbb{R}^d$
    \item $W_{c}=\left[{W_{c}^{(1)}}^T,\dots,{W_{c}^{(K)}}^T\right]^T\in\mathbb{R}^{d\times q}$
    \item $\boldsymbol{\mu}_{c}=\left[{\boldsymbol{\mu}_{c}^{(1)}}^T,\dots,{\boldsymbol{\mu}_{c}^{(K)}}^T\right]^T\in\mathbb{R}^{d}$
    \item $\boldsymbol{\varepsilon}_{c}=\left[{\boldsymbol{\varepsilon}_{c}^{(1)}}^T,\dots,{\boldsymbol{\varepsilon}_{c}^{(K)}}^T\right]^T\sim\mathcal{N}(0,\Psi_{c})$,\\ where $\Psi_{c}$ is a diagonal block-matrix, $\Psi_{c}=diag\left({\sigma_{c}^{(1)}}^2\mathbb{I}_{d_1},\dots,{\sigma_{c}^{(K)}}^2\mathbb{I}_{d_K}\right)$
\end{itemize}
Note that for the sake of simplicity we represented all $K$ views. If in center $c$ the $k^{\textrm{th}}$-view is missing, than it will be simply removed, \emph{e.g.} one would have:\\ $\mathbf{t}_{c,n}=\left[{\mathbf{t}_{c,n}^{(1)}}^T,\dots,{\mathbf{t}_{c,n}^{(k-1)}}^T,{\mathbf{t}_{c,n}^{(k+1)}}^T,\dots,{\mathbf{t}_{c,n}^{(K)}}^T\right]^T$, $\mathbf{t}_{c,n}\in\mathbb{R}^{d-d_k}$.\\

For each center $c$ and each $k$ we want to estimate $\boldsymbol{\theta}_{c}:=\left\{\boldsymbol{\mu}_{c}^{(k)},W_{c}^{(k)},{\sigma_{c}^{(k)}}^2\right\}_{k=1,\dots,K_c}$ assuming that all local parameters are a realization of a common global distribution, to be estimated as well. The latter, provide a global model, which should be able to describe data across all centers.

\subsection*{Parameter $\boldsymbol{\mu}$}

We assume that $\forall c, k$:
\begin{equation}\label{dist_mukc}
\boldsymbol{\mu}_{c}^{(k)}|\widetilde{\boldsymbol{\mu}}^{(k)},\sigma_{\widetilde{\boldsymbol{\mu}}^{(k)}}^2\sim\mathcal{N}\left(\widetilde{\boldsymbol{\mu}}^{(k)},\sigma_{\widetilde{\boldsymbol{\mu}}^{(k)}}^2\mathbb{I}_{d_k}\right)
\end{equation}

\textbf{Step 1. \textit{(In each center)}:} Estimate $\boldsymbol{\mu}_{c}^{(k)}[s+1]$ given $\left(\widetilde{\boldsymbol{\mu}}^{(k)}, \sigma_{\widetilde{\boldsymbol{\mu}}^{(k)}}^2\right)[s]$ (iteration $s$ is denoted by $[s]$).\\

From Equation \eqref{tcnkg_supp}, the marginal distribution of $\mathbf{t}_{c,n}^{(k),g}$ is:
\begin{equation*}
\mathbf{t}_{c,n}^{(k)}\sim\mathcal{N}(\boldsymbol{\mu}_{c}^{(k)},C_{c}^{(k)}),
\end{equation*}
where $C_{c}^{(k)}=W_{c}^{(k)}{W_{c}^{(k)}}^T+{\sigma_{c}^{(k)}}^2\mathbb{I}_{d_k}$, $C_{c}^{(k)}\in\mathbb{R}^{d_k\times d_k}$. 

The corresponding log-likelihood gives:
\begin{eqnarray}
    \mathcal{L}_{c}^{(k)} & = &  -\frac{1}{2}\left\{N_{c}d_k\ln{(2\pi)}+N_{c}\ln{|C_{c}^{(k)}|}+\sum_{n=1}^{N_{c}}\left(\mathbf{t}_{c,n}^{(k)}-\boldsymbol{\mu}_{c}^{(k)}\right)^T\left(C_{c}^{(k)}\right)^{-1}\left(\mathbf{t}_{c,n}^{(k)}-\boldsymbol{\mu}_{c}^{(k)}\right)\right\}\nonumber \\
\end{eqnarray}
Therefore, for each center $c$ and for all $k\in\{1,\dots,K\}$, the following optimization problem should be considered:
\begin{equation*}
\max_{\boldsymbol{\mu}_{c}^{(k)}}\mathcal{L}_{c}^{(k)}+\ln{p\left(\boldsymbol{\mu}_{c}^{(k)}\right)},
\end{equation*}
where: 
\begin{equation*}
\ln{p\left(\boldsymbol{\mu}_{c}^{(k)}\right)}=-\frac{1}{2\sigma_{\widetilde{\boldsymbol{\mu}}^{(k)}}^2}\left(\boldsymbol{\mu}_{c}^{(k)}-\widetilde{\boldsymbol{\mu}}^{(k)}\right)^T\left(\boldsymbol{\mu}_{c}^{(k)}-\widetilde{\boldsymbol{\mu}}^{(k)}\right)+const,
\end{equation*}
where $const$ collects terms which are independents from $\boldsymbol{\mu}_{c}^{(k)}$.
We obtain:
\begin{eqnarray}\label{muck[s]}
    \mu_{c}^{(k)}[s+1] & = & \left[N_{c}\mathbb{I}_{d_k}+\frac{1}{\sigma_{\widetilde{\boldsymbol{\mu}}^{(k)}}^2[s]}C_{c}^{(k)}\right]^{-1}\left[\sum_{n=1}^{N_{c}}\mathbf{t}_{c,n}^{(k)}+\frac{1}{\sigma_{\widetilde{\boldsymbol{\mu}}^{(k)}}^2[s]}C_{c}^{(k)}\widetilde{\boldsymbol{\mu}}^{(k)}[s]\right] \nonumber \\
\end{eqnarray}

\textbf{Step 2. \textit{(In the master)}:} Estimate $\left(\widetilde{\boldsymbol{\mu}}^{(k)}[s+1], \sigma_{\widetilde{\boldsymbol{\mu}}^{(k)}}^2\right)[s+1]$ given $\boldsymbol{\mu}_{c}^{(k)}[s+1]$ for all $c$. \\

Using \eqref{dist_mukc}, we obtain the following log-likelihood:
\begin{equation}\label{log_like_tildemu}
    \mathcal{L}=\sum_{c=1}^C\ln{p(\boldsymbol{\mu}_{c}^{(k)})}=\sum_{c=1}^C\left\{const-\frac{1}{\sigma_{\widetilde{\boldsymbol{\mu}}^{(k)}}^2}\|\boldsymbol{\mu}_{c}^{(k)}-\widetilde{\boldsymbol{\mu}}^{(k)}\|^2\right\}
\end{equation}
By imposing $\partial_{\left(\widetilde{\boldsymbol{\mu}}^{(k)},\sigma_{\widetilde{\boldsymbol{\mu}}^{(k)}}^2\right)}\left(\eqref{log_like_tildemu}\right)=0$ we obtain:
\begin{equation}
    \widetilde{\boldsymbol{\mu}}^{(k)}[s+1]=\frac{1}{C}\sum_{c=1}^C\boldsymbol{\mu}_{c}^{(k)}[s+1]
\end{equation}
and
\begin{equation}
    \sigma_{\widetilde{\boldsymbol{\mu}}^{(k)}}^2[s+1]=\frac{1}{Cd_k}\sum_{c=1}^C\left\|\boldsymbol{\mu}_{c}^{(k)}[s+1]-\widetilde{\boldsymbol{\mu}}^{(k)}[s+1]\right\|^2
\end{equation}

\subsection*{Complete-data log-likelihood}
From Equations \eqref{tcnkg_supp}-\eqref{tcng_supp} one can derive the following marginal distributions:
\begin{eqnarray*}
    \mathbf{t}_{c,n}^{(k)}|\mathbf{x}_{c,n}\sim\mathcal{N}\left(W_{c}^{(k)}\mathbf{x}_{c,n}+\boldsymbol{\mu},{\sigma_{c}^{(k)}}^2\mathbb{I}_{d_k}\right)
    \end{eqnarray*}
and
    \begin{eqnarray*}\mathbf{x}_{c,n}|\mathbf{t}_{c,n}\sim\mathcal{N}\left(\Sigma_{c}^{-1}B_{c}(\mathbf{t}_{c,n}-\boldsymbol{\mu}_{c}),\Sigma_{c}^{-1}\right),
    \end{eqnarray*} where:
\begin{itemize}
    \item $\Sigma_{c}:=(\mathbb{I}_q+W_{c}^T\Psi_{c}^{-1}W_{c})=\left(\mathbb{I}_q+\sum_{k=1}^K\frac{1}{\left(\sigma_{c}^{(k)}\right)^2}{W_{c}^{(k)}}^TW_{c}^{(k)}\right)\in\mathbb{R}^{q\times q}$
    \item $B_{c}:={W_{c}}^T\Psi_{c}^{-1}=\left[\frac{{W_{c}^{(1)}}^T}{\left(\sigma_{c}^{(1)}\right)^2}\dots,\frac{{W_{c}^{(K)}}^T}{\left(\sigma_{c}^{(K)}\right)^2}\right]\in\mathbb{R}^{q\times d}$
\end{itemize}
Hence:
\begin{itemize}
    \item $\langle \mathbf{x}_{c,n}\rangle=\Sigma_{c}^{-1}B_{c}(\mathbf{t}_{c,n}-\boldsymbol{\mu}_{c})$
    \item $\langle \mathbf{x}_{c,n}\mathbf{x}_{c,n}^T\rangle=\Sigma_{c}^{-1}+\langle \mathbf{x}_{c,n}\rangle\langle \mathbf{x}_{c,n}\rangle^T$
\end{itemize}

The joint distribution of $\mathbf{t}_{c,n}$ and $\mathbf{x}_{c,n}$ follows ($p(\mathbf{t}_{c,n},\mathbf{x}_{c,n})=p(\mathbf{t}_{c,n}|\mathbf{x}_{c,n})p(\mathbf{x}_{c,n})$), hence the expectation of the complete-data log-likelihood for each center $c$  with respect to $p(\mathbf{x}_{c,n}|\mathbf{t}_{c,n})$:
\begin{eqnarray}
\langle{\mathcal{L}_C}_{c}\rangle & = & -\sum_{n=1}^{N_{c}}\left\{\sum_{k=1}^K\left[\frac{d_k}{2}\ln{\left({\sigma_{c}^{(k)}}^2\right)}+\frac{1}{2{\sigma_{c}^{(k)}}^2}    \|\mathbf{t}_{c,n}^{(k)}-\boldsymbol{\mu}_{c}^{(k)}\|^2+\frac{1}{2{\sigma_{c}^{(k)}}^2}tr\left({W_{c}^{(k)}}^T{W_{c}^{(k)}}\langle \mathbf{x}_{c,n}\mathbf{x}_{c,n}^T\rangle\right)\right.\right. \nonumber \\
    & & \left.\left.-\frac{1}{{\sigma_{c}^{(k)}}^2}\langle \mathbf{x}_{c,n}\rangle^T{W_{c}^{(k)}}^T\left(\mathbf{t}_{c,n}^{(k)}-\boldsymbol{\mu}_{c}^{(k)}\right)\right]+\frac{1}{2}tr\left(\langle \mathbf{x}_{c,n}\mathbf{x}_{c,n}^T\rangle\right)\right\},
\end{eqnarray}

\subsection*{Parameter $W$}

We assume that $\forall c, k$: 
\begin{equation}\label{dist_Wkc}
W_{c}^{(k)}|\widetilde{W}^{(k)},\sigma_{\widetilde{W}^{(k)}}^2\sim\mathcal{MN}_{d_k,q}\left(\widetilde{W}^{(k)},\mathbb{I}_{d_k},\sigma_{\widetilde{W}^{(k)}}^2\mathbb{I}_q\right)
\end{equation}

\textbf{Step 1. \textit{(In each center)}:} Estimate $W_{c}^{(k)}[s+1]$ given $\left(\widetilde{W}^{(k)},\sigma_{\widetilde{W}^{(k)}}^2\right)[s]$.\\

For each center $c$, we consider the following optimization problem:
\begin{equation*}
\max_{W_{c}^{(k)}}\langle{\mathcal{L}_C}_{c}\rangle+\ln{p\left(W_{c}^{(k)}\right)},
\end{equation*}
where $\ln{p\left(W_{c}^{(k)}\right)}=-\frac{1}{2\sigma_{\widetilde{W}^{(k)}}^2}tr\left(\|W_{c}^{(k)}-\widetilde{W}^{(k)}\|_2^2\right)+const.$ \\
% We recall that if $X\sim\mathcal{MN}_{n,p}(M,U,V)$, with $M\in\mathbb{R}^{n\times p}$, $U\in\mathbb{R}^{n\times n}$ and $V\in\mathbb{R}^{p\times p}$, then:
% \begin{equation*}
%     p(X|M,U,V)=(2\pi)^{-np/2}|V|^{-n/2}|U|^{-p/2}\exp{\left(-\frac{1}{2}tr\left[V^{-1}(X-M)^T U^{-1}(X-M)\right]\right)}
% \end{equation*}
It follows:
\begin{eqnarray*}
    W_{c}^{(k)}[s+1] & = & \left[\sum_{n=1}^{N_{c}}(\mathbf{t}_{c,n}^{(k)}-\boldsymbol{\mu}_{c}^{(k)})\langle \mathbf{x}_{c,n}\rangle^T+\frac{{\sigma_{c}^{(k)}}^2}{\sigma_{\widetilde{W}^{(k)}}^2[s]} \widetilde{W}^{(k)}[s]\right]\left[\sum_{n=1}^{N_{c}}\langle \mathbf{x}_{c,n}\mathbf{x}_{c,n}^T\rangle
    +\frac{{\sigma_{c}^{(k)}}^2}{\sigma_{\widetilde{W}^{(k)}}^2[s]}\mathbb{I}_q\right]^{-1}\nonumber\\
\end{eqnarray*}

\textbf{Step 2. \textit{(In the master)}:} Estimate $\left(\widetilde{W}^{(k)},\sigma_{\widetilde{W}^{(k)}}^2\right)[s+1]$ given $W_{c}^{(k)}[s+1]$ for all $c$. Proceeding as for parameter $\boldsymbol{\mu}$ and using \eqref{dist_Wkc}:
\begin{equation}
    \widetilde{W}^{(k)}[s+1]=\frac{1}{C}\sum_{c=1}^CW_{c}^{(k)}[s+1]
\end{equation}
and
\begin{eqnarray}
    \sigma_{\widetilde{W}^{(k)}}^2[s+1] & = &\frac{1}{Cd_kq}\sum_{c=1}^Ctr\left[\left(W_{c}^{(k)}[s+1]-\widetilde{W}^{(k)}[s+1]\right)^T\left(W_{c}^{(k)}[s+1]-\widetilde{W}^{(k)}[s+1]\right)\right]\nonumber\\
    & & 
\end{eqnarray}

\subsection*{Parameter $\sigma^2$}

We assume that $\forall c, k$:
\begin{equation}
    {\sigma_{c}^{(k)}}^2|\widetilde{\sigma}^{(k)^2}\sim \textrm{Inverse-Gamma}(\alpha^{(k)},\beta^{(k)}),
\end{equation}
so that:
\begin{equation}
    Var\left({\sigma_{c}^{(k)}}^2\right)=\frac{{\beta^{(k)}}^2}{(\alpha^{(k)}-1)^2(\alpha^{(k)}-2)}:={\widetilde{\sigma}^{(k)^2}}
\end{equation}
% Then, in particular:
% \begin{equation}
%     \mathbb{E}\left[\left(\sigma_{c}^{(k)}\right)^2\right]=\left(\widetilde{\sigma}^{(k)}\right)^2=\frac{\beta^{(k)}}{\alpha^{(k)}-1}\textrm{, and }\mathrm{Var}\left[\left(\sigma_{c}^{(k)}\right)^2\right]=\frac{{\beta^{(k)}}^2}{(\alpha^{(k)}-1)^2(\alpha^{(k)}-2)}
% \end{equation}

\textbf{Step 1. \textit{(In each center)}:} Estimate ${\sigma_{c}^{(k)}}^2[s+1]$ given $\left(\alpha^{(k)},\beta^{(k)}\right)[s]$. \\

For each center $c$, we consider the following optimization problem:
\begin{equation*}
\max_{{\sigma_{c}^{(k)}}^2}\langle{\mathcal{L}_C}_{c}\rangle+\ln{ p\left({\sigma_{c}^{(k)}}^2\right)},
\end{equation*}
where $\ln{ p\left({\sigma_{c}^{(k)}}^2\right)}=-(\alpha^{(k)}+1)\ln{\left({\sigma_{c}^{(k)}}^2\right)}-\frac{\beta^{(k)}}{{\sigma_{c}^{(k)}}^2}+const$: 

It follows:
\begin{eqnarray}
    \left(\sigma_{c}^{(k)}\right)^2[s+1] & = &\frac{1}{N_{c}d_k+2(\alpha^{(k)}[s]+1)}\left\{\sum_{n=1}^{N_{c}}\left[\|t_{c,n}^{(k)}-\mu_{c}^{(k)}\|^2+ tr\left({W_{c}^{(k)}}^T{W_{c}^{(k)}}\langle x_{c,n}x_{c,n}^T\rangle\right)\right.\right.\nonumber\\
    & & \left.\left.- 2\langle x_{c,n}\rangle^T{W_{c}^{(k)}}^T\left(t_{c,n}^{(k)}-\mu_{c}^{(k)}\right) \right]+2\beta^{(k)}[s]\right\}\nonumber\\
     & = &\frac{1}{N_{c}d_k+2(\alpha^{(k)}[s]+1)}\left\{\sum_{n=1}^{N_{c}}\left[\|(t_{c,n}^{(k)}-\mu_{c}^{(k)})-W_{c}^{(k)}\langle x_{c,n}\rangle\|^2\right.\right.\nonumber\\
    & & \left.\left.+ tr\left(W_{c}^{(k)}\Sigma_c^{-1}{W_{c}^{(k)}}^T\right) \right]+2\beta^{(k)}[s]\right\}
\end{eqnarray}

\textbf{Step 2. \textit{(In the master)}:} Estimate $\left(\alpha^{(k)}, \beta^{(k)}\right)[s+1]$ given ${\sigma_{c}^{(k)}}^2[s+1]$ for all $c$. \\

In order to estimate the parameters of the inverse-gamma distribution, we use the (ML1) method described by Llera and Beckmann~\cite{llera2016estimating}.

\section*{Appendix B. Supplementary Tables and Figures}\label{AppendixB}%{Tables and Figures}

\begin{table}[ht]
\centering
\caption{Demographics of the clinical sample from the Alzheimer's Disease Neuroimaging Initiative
 (ADNI).}\label{Tab_data1}
\label{tab:demographics}
\centering
\begin{tabular}{llccc}
\textbf{Group} & \textbf{Sex} & \textbf{Count}  & \textbf{Age}   & \textbf{Range}\\ 
\hline \\
\multirow{ 2}{*}{AD}   & Female  & 94     & 71.58 (7.59)    & 55.10 - 90.30   \\
   & Male         & 113    & 74.37 (7.19)     & 55.90 - 89.30    \\ \\
\multirow{ 2}{*}{NL}    & Female    & 58      & 73.76 (4.61)   & 65.10 - 84.70      \\
  & Male  & 46          & 75.39 (6.58)     & 59.90 - 85.60    
\end{tabular}
\end{table}

\begin{table}[ht]
\caption{Data Types.}\label{Tab_data2}
\label{tab:views}
\centering
\begin{tabular}{lcp{5cm}}
\textbf{View} & \textbf{Dim.} & \textbf{Description}  \\ 
\hline \\
CLINIC & 7 & Cognitive assessments \\
MRI & 41 & Magnetic resonance imaging \\
FDG & 41 & Fluorodeoxyglucose-Positron Emission Tomography (PET) \\
AV45 & 41 & AV45-Amyloid PET 
\end{tabular}
\end{table}

\begin{table}[ht]
\caption{Latent Space Dimension Assessment.} \label{Latent_space}
\centering
\begin{tabular}{c|c|c|c|c|c|c}
 & \multicolumn{3}{|c|}{\textbf{SD}} & \multicolumn{3}{|c}{\textbf{ADNI}}\\
\hline & & & & & \\
$\mathbf{q}$  & \textbf{WAIC} & \textbf{MAE Train} & \textbf{MAE Test} & \textbf{WAIC} & \textbf{MAE Train} & \textbf{MAE Test} \\
\hline & & & & & \\
2 & -2911 & 0.0886$\pm$0.0031 & 0.0879$\pm$0.0024 & -4916 & 0.1240$\pm$0.0017 & 0.1249$\pm$0.0011 \\
3 & -3954 & 0.0640$\pm$0.0029 & 0.0662$\pm$0.0038 & -5275 & 0.1170$\pm$0.0028 & 0.1187$\pm$0.0023 \\
4 & -4725 & 0.0450$\pm$0.0036 & 0.0485$\pm$0.0042 & -6088 & 0.1113$\pm$0.0016 & 0.1142$\pm$0.0009 \\
5 & $\mathbf{-5114}$ & 0.0327$\pm$0.0038 & 0.0375$\pm$0.0054 & -6915 & 0.1064$\pm$0.0017 & 0.1102$\pm$0.0007 \\
6 & -3688 & 0.0313$\pm$0.0032 & 0.0366$\pm$0.0052  & $\mathbf{-7546}$ & 0.1028$\pm$0.0015 & 0.1073$\pm$0.0005\\
7 & 5722 & 0.0320$\pm$0.0036 & 0.0373$\pm$0.0053 & - & - & - \\
\end{tabular}
\end{table}

\begin{table}[ht]
\caption{Results on SD dataset for all scenarios, and comparison with VAE and mc-VAE.} \label{SD_2}
\centering
\begin{tabular}{l|c|l|l|l|l}
\textbf{Scenario} & \textbf{Centers} & \textbf{Method} & \textbf{MAE Train} & \textbf{MAE Test} & \textbf{Accuracy in LS}\\
\hline 
\multirow{ 12}{*}{IID} & \multirow{ 3}{*}{\makecell{1 \\ (centralized \\ case)}} & \textbf{Fed-mv-PPCA} & \textbf{0.0124$\boldsymbol\pm$3.e$^{\mathbf{-5}}$} & \textbf{0.0405$\boldsymbol\pm$0.0037} & 1$\pm$0\\
% & & \textbf{PPCA} & \textbf{0.0124$\boldsymbol\pm$3.e$^{\mathbf{-5}}$} & \textbf{0.0404$\boldsymbol\pm$0.0037} & 1$\pm$0\\
& & VAE & 0.0851$\pm$0.0039 & 0.1011$\pm$0.0048 & 1$\pm$0\\
& & mc-VAE & 0.1236$\pm$0.0099 & 0.1382$\pm$0.0087 & 1$\pm$0\\\cline{2-6}
% & \multirow{3}{*}{2} & \textbf{Fed-mv-PPCA} & \textbf{0.0317$\boldsymbol\pm$0.0063} & \textbf{0.0426$\boldsymbol\pm$0.0055} & 1$\pm$0\\
% & & VAE & 0.0621$\pm$0.0282 & 0.0688$\pm$0.0286 & 1$\pm$0\\
% & & mc-VAE & 0.1091$\pm$0.0310 & 0.1136$\pm$0.0326 & 1$\pm$	0\\\cline{2-6}
& \multirow{ 4}{*}{3} & \textbf{Fed-mv-PPCA} & \textbf{0.0320$\boldsymbol\pm$0.0024} & \textbf{0.0373$\boldsymbol\pm$0.0035} & 1$\pm$0\\
& & DP-Fed-mv-PPCA & 0.0858$\pm$0.0111 & 0.0848$\pm$0.0099 & 1$\pm$0\\
& & VAE & 0.0683$\pm$0.0073 & 0.0702$\pm$0.0073 & 1$\pm$0\\
& & mc-VAE & 0.1172$\pm$0.0030 & 0.1146$\pm$0.0046 & 1$\pm$0\\\cline{2-6}
% & \multirow{ 3}{*}{4} & \textbf{Fed-mv-PPCA} & \textbf{0.0399$\boldsymbol\pm$0.0051} & \textbf{0.0404$\boldsymbol\pm$0.0064} & 1$\pm$0\\
% & & VAE & 0.0809$\pm$0.0378 & 0.0769$\pm$0.0328 & 1$\pm$0\\
% & & mc-VAE & 0.1257$\pm$0.0386 & 0.1163$\pm$0.0328 & 1$\pm$0\\\cline{2-6}
% & \multirow{ 3}{*}{5} & \textbf{Fed-mv-PPCA} & \textbf{0.0411$\boldsymbol\pm$0.0044} & \textbf{0.0389$\boldsymbol\pm$0.0030} & 1$\pm$0\\
% & & VAE & 0.0790$\pm$0.0358 & 0.0720$\pm$0.0299 & 1$\pm$0\\
% & & mc-VAE & 0.1281$\pm$0.0410 & 0.1147$\pm$0.0324 & 1$\pm$0\\\cline{2-6}
& \multirow{ 4}{*}{6} & \textbf{Fed-mv-PPCA} & \textbf{0.0422$\boldsymbol\pm$0.0052} & \textbf{0.0371$\boldsymbol\pm$0.0039} & 1$\pm$0\\
& & DP-Fed-mv-PPCA & 0.0843$\pm$0.0093 & 0.0738$\pm$0.0076 & 1$\pm$0\\
& & VAE & 0.0769$\pm$0.0093 & 0.0680$\pm$0.0080 & 1$\pm$0\\
& & mc-VAE & 0.1295$\pm$0.0055 & 0.1134$\pm$0.0030 & 1$\pm$0\\
\hline 
\multirow{ 8}{*}{G} & \multirow{ 4}{*}{3} & \textbf{Fed-mv-PPCA} & \textbf{0.0432$\boldsymbol\pm$0.0074} & \textbf{0.0433$\boldsymbol\pm$0.0026} & \textbf{0.9930$\boldsymbol\pm$0.0093}\\
& & DP-Fed-mv-PPCA & 0.0960$\pm$0.0151 & 0.0951$\pm$0.0144 & 0.9873$\pm$0.0176\\
& & VAE & 0.0787$\pm$0.0135 & 0.0698$\pm$0.0082 & 0.9835$\pm$0.0272\\
& & mc-VAE & 0.1562$\pm$0.0086 & 0.1497$\pm$0.0076 & 0.9732$\pm$0.0512\\\cline{2-6}
& \multirow{ 4}{*}{6} & \textbf{Fed-mv-PPCA} & \textbf{0.0538$\boldsymbol\pm$0.0101} & \textbf{0.0420$\boldsymbol\pm$0.0048} & \textbf{0.9995$\boldsymbol\pm$0.0019}\\
& & DP-Fed-mv-PPCA & 0.0945$\pm$0.0129 & 0.0813$\pm$0.0114 & 1$\pm$0\\
& & VAE & 0.0891$\pm$0.0148 & 0.0685$\pm$0.0063 & 0.9918$\pm$0.0428\\
& & mc-VAE & 0.1758$\pm$0.0154 & 0.1495$\pm$0.0112 & 0.9607$\pm$0.0398\\
\hline\hline
\multirow{ 4}{*}{K} & \multirow{ 2}{*}{3} & Fed-mv-PPCA & 0.0320$\pm$0.0052 & 0.0455$\pm$0.0069 & 1$\pm$0 \\
& & DP-Fed-mv-PPCA & 0.0922$\pm$0.0137 & 0.1048$\pm$0.0151 & 1$\pm$0\\\cline{2-6}
& \multirow{ 2}{*}{6} & Fed-mv-PPCA & 0.0402$\pm$0.0065 & 0.0448$\pm$0.0088 & 1$\pm$0 \\
& & DP-Fed-mv-PPCA & 0.0959$\pm$0.0105 & 0.1014$\pm$0.0119 & 1$\pm$0\\
\hline 
\multirow{ 4}{*}{G/K} & \multirow{ 2}{*}{3} & Fed-mv-PPCA & 0.0395$\pm$0.0068 & 0.0567$\pm$0.0108 & 0.7812$\pm$0.02179 \\
& & DP-Fed-mv-PPCA & 0.1144$\pm$0.0215 & 0.1343$\pm$0.0235 & 0.7852$\pm$0.0526\\\cline{2-6}
& \multirow{ 2}{*}{6} & Fed-mv-PPCA & 0.0499$\pm$0.0104 & 0.0575$\pm$0.0128 & 0.7785$\pm$0.0222 \\
& & DP-Fed-mv-PPCA & 0.1070$\pm$0.0139 & 0.1119$\pm$0.0144 & 0.7887$\pm$0.0449
\end{tabular}
\end{table}

\clearpage

\begin{figure}
     \begin{subfigure}[b]{0.4\textwidth}
         \centerline{\includegraphics[width=\textwidth]{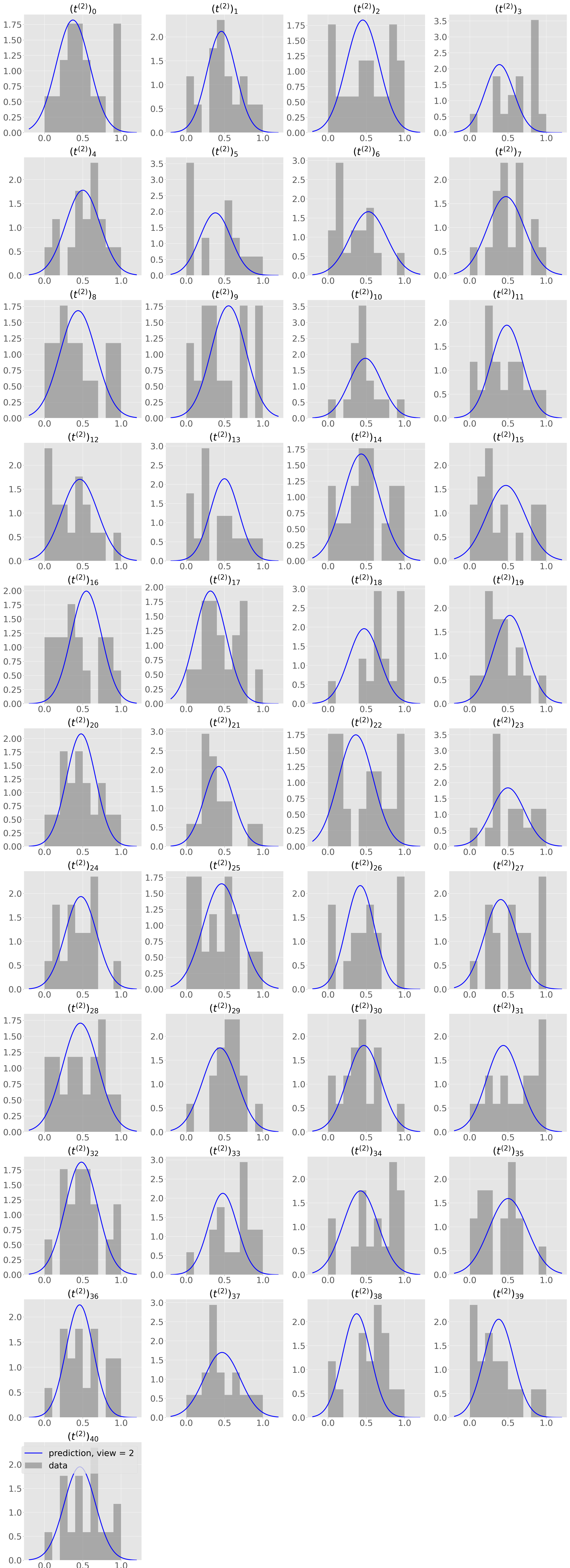}}
         \caption{MRI}
     \end{subfigure}
     \hfill
     \begin{subfigure}[b]{0.4\textwidth}
         \centerline{\includegraphics[width=\textwidth]{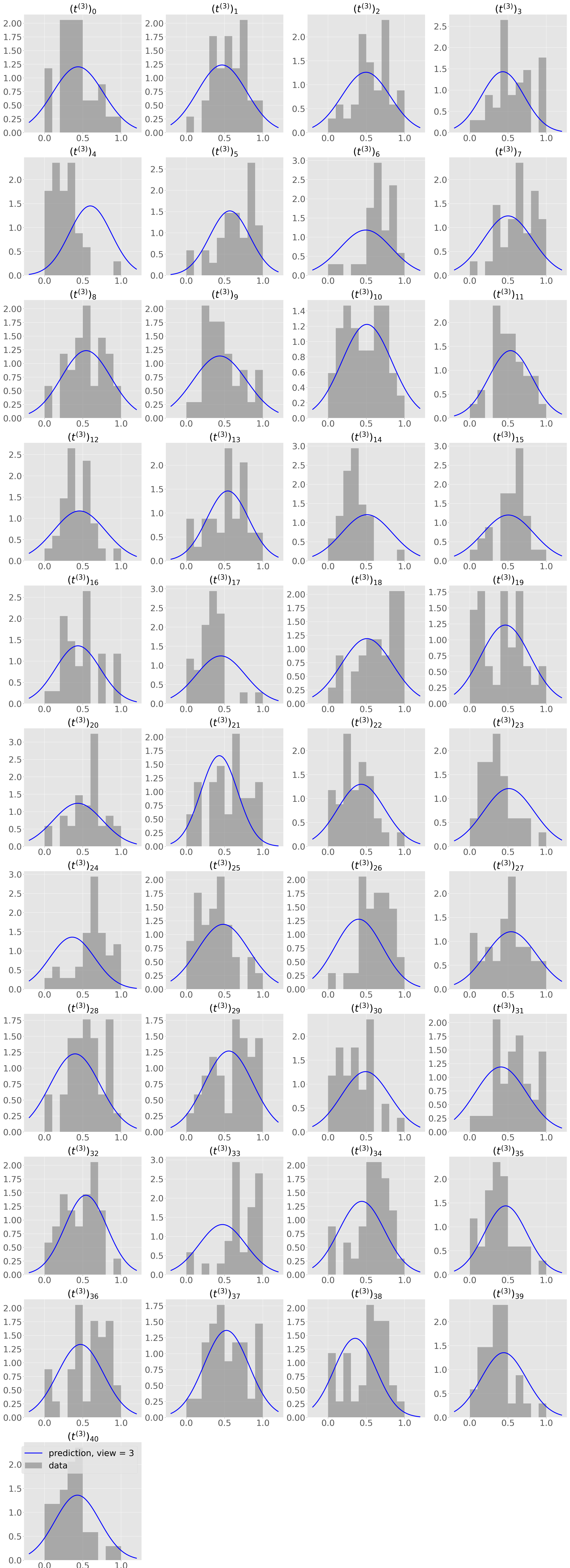}}
         \caption{FDG}
     \end{subfigure}
\caption{Global distribution of all features of missing views in the Test dataset, for the G/K scenario. In this scenario, 1/3 of all subjects in the Test dataset do not provide MRI data and 1/3 do not provide FDG data. In both figure, the blue curve denotes the predicted global distribution  of all features of the (a) MRI view and (b) the FDG view. Gray histograms correspond to real data in the Test dataset.}\label{Prediction_ADNI_supp}
\end{figure}

\begin{table}[ht]
\caption{Results on ADNI dataset for all scenario G using VAE (resp. mc-VAE), and FedProx as aggregation scheme with the proximal term $\lambda$ varying from 0.01 to 0.5.} \label{ADNI_fedprox}
\centering
\begin{tabular}{l|c|l|l|l|l}
\textbf{Centers} & \textbf{Method} & $\mathbf{\lambda}$ & \textbf{MAE Train} & \textbf{MAE Test} & \textbf{Accuracy in LS}\\
\hline %& & & & \\
% G, C=3
\multirow{ 16}{*}{3} & \multirow{8}{*}{VAE} & 0 (FedAvg) & 0.1172$\pm$0.0022 & 0.1192$\pm$0.0015 & 0.8289$\pm$0.0383\\
& & 0.01 & 0.1209$\pm$0.0074 & 0.1215$\pm$0.0013 & 0.7962$\pm$0.0438\\
& & 0.05 & 0.1215$\pm$0.0076 & 0.1218$\pm$0.0015 & 0.8009$\pm$0.0425\\
& & 0.1 & 0.1214$\pm$0.0075 & 0.1220$\pm$0.0018 & 0.8067$\pm$0.0399\\
& & 0.2 & 0.1218$\pm$0.0077 & 0.1221$\pm$0.0016 & 0.7977$\pm$0.0469\\
& & 0.3 & 0.1212$\pm$0.0075 & 0.1216$\pm$0.0017 & 0.7865$\pm$0.0443\\
& & 0.4 & 0.1212$\pm$0.0074 & 0.1217$\pm$0.0014 & 0.8033$\pm$0.0355\\
& & 0.5 & 0.1214$\pm$0.0077 & 0.1218$\pm$0.0020 & 0.7878$\pm$0.0420\\\cline{2-6}
& \multirow{8}{*}{mc-VAE} & 0 (FedAvg) & 0.1602$\pm$0.0035 & 0.1567$\pm$0.0017 & 0.8850$\pm$0.0262\\
& & 0.01 &  0.1674$\pm$0.0155 & 0.1605$\pm$0.0028 & 0.8185$\pm$0.0494\\
& & 0.05 &  0.1667$\pm$0.0153 & 0.1604$\pm$0.0028 & 0.8156$\pm$0.0444\\
& & 0.1 & 0.1674$\pm$0.0154 & 0.1609$\pm$0.0022 & 0.8249$\pm$0.0399\\
& & 0.2 & 0.1676$\pm$0.0156 & 0.1610$\pm$0.0025 & 0.8217$\pm$0.0431\\
& & 0.3 & 0.1676$\pm$0.0157 & 0.1610$\pm$0.0029 & 0.8184$\pm$0.0511\\
& & 0.4 & 0.1673$\pm$0.0155 & 0.1607$\pm$0.0021 & 0.8275$\pm$0.0426\\
& & 0.5 & 0.1679$\pm$0.0157 & 0.1613$\pm$0.0025 & 0.8229$\pm$0.0408\\
\hline
% G, C=6
\multirow{ 16}{*}{6} & \multirow{8}{*}{VAE} & 0 (FedAvg) & 0.1357$\pm$0.0042 & 0.1191$\pm$0.0014 & 0.8224$\pm$0.0377\\
& & 0.01 & 0.1400$\pm$0.0114 & 0.1198$\pm$0.0022 & 0.7804$\pm$0.0470\\
& & 0.05 & 0.1403$\pm$0.0115 & 0.1203$\pm$0.0021 & 0.7827$\pm$0.0411\\
& & 0.1 & 0.1406$\pm$0.0116 & 0.1205$\pm$0.0019 & 0.7847$\pm$0.0531\\
& & 0.2 & 0.1407$\pm$0.0117 & 0.1207$\pm$0.0018 & 0.7837$\pm$0.0433\\
& & 0.3 & 0.1404$\pm$0.0115 & 0.1207$\pm$0.0018 & 0.7837$\pm$0.0569\\
& & 0.4 & 0.1405$\pm$0.0116 & 0.1203$\pm$0.0020 & 0.7753$\pm$0.0546\\
& & 0.5 & 0.1406$\pm$0.0113 & 0.1205$\pm$0.0023 & 0.7776$\pm$0.0501\\\cline{2-6}
& \multirow{8}{*}{mc-VAE} & 0 (FedAvg) & 0.1840$\pm$0.0054 & 0.1563$\pm$0.0017 & 0.8894$\pm$0.0230\\
& & 0.01 & 0.1932$\pm$0.0220 & 0.1596$\pm$0.0019 & 0.8140$\pm$0.0420\\
& & 0.05 & 0.1927$\pm$0.0219 & 0.1592$\pm$0.0016 & 0.8101$\pm$0.0484\\
& & 0.1 & 0.1932$\pm$0.0221 & 0.1595$\pm$0.0022 & 0.8043$\pm$0.0399\\
& & 0.2 & 0.1930$\pm$0.0219 & 0.1596$\pm$0.0020 & 0.8066$\pm$0.0441\\
& & 0.3 & 0.1931$\pm$0.0221 & 0.1595$\pm$0.0019 & 0.8217$\pm$0.0453\\
& & 0.4 & 0.1931$\pm$0.0220 & 0.1594$\pm$0.0018 & 0.8111$\pm$0.0419\\
& & 0.5 & 0.1934$\pm$0.0221 & 0.1596$\pm$0.0022 & 0.8021$\pm$0.0581\\
\end{tabular}
\end{table}

\begin{figure}
% \begin{subfigure}[b]{0.35\textwidth}
%          \centerline{\includegraphics[width=0.8\textwidth]{MELBA/Figures/WAIC_std_mean_diff.png}}
%          \caption{Standardized mean differences}
%      \end{subfigure}
%      \hfill
     \begin{subfigure}[b]{0.5\textwidth}
         \centerline{\includegraphics[width=\textwidth]{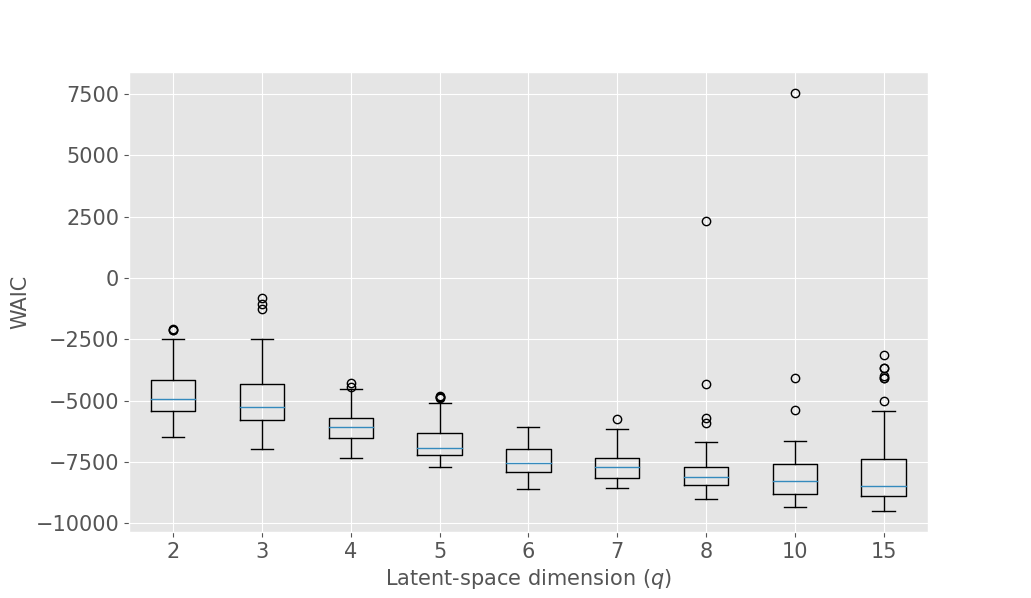}}
         \caption{$2\leq q\leq 15$, boxplot}
     \end{subfigure}
    %  \hfill
    %  \begin{subfigure}[b]{0.35\textwidth}
    %      \centerline{\includegraphics[width=\textwidth]{MELBA/Figures/ADNI_IID_q_MAE_WAIC.png}}
    %      \caption{$2\leq q\leq 15$, mean and std}
    %  \end{subfigure}
     \hfill
     \begin{subfigure}[b]{0.38\textwidth}
         \centerline{\includegraphics[width=\textwidth]{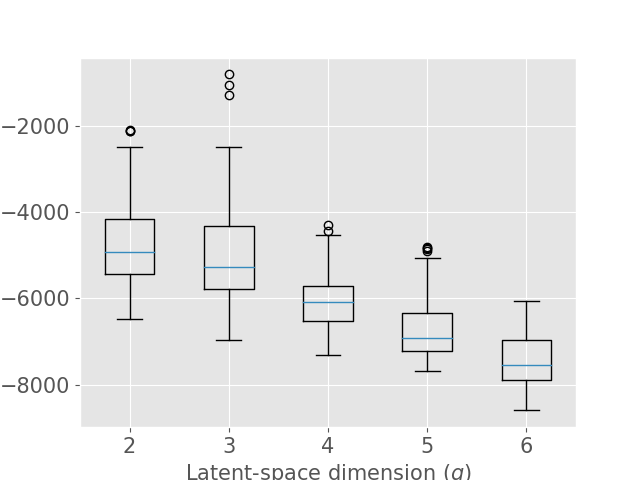}}
         \caption{$2\leq q\leq 6$, boxplot}
     \end{subfigure}
\caption{Boxplots showing the evolution of the WAIC with the latent space dimension varying (a) from 2 to 15 and (b) a zoom with latent dimension $q$ up to 6. We notice that an increasing complexity of the model provide only a relative improvement of the median WAIC score, while being associated to a higher variation suggesting less stable results.}\label{fig_waic_q_15}
\end{figure}

\begin{figure}[ht]
\centering
\includegraphics[width=\textwidth]{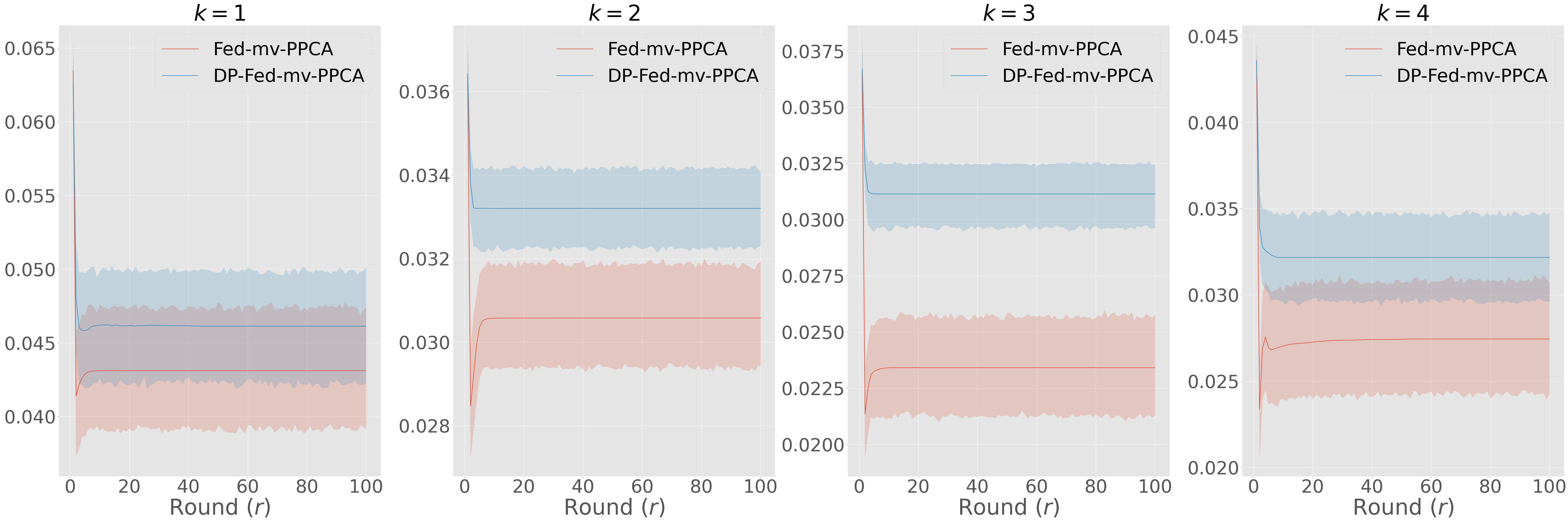}
\caption{Mean and standard deviation of $\widetilde{\sigma}^{(k)^2}$ at each round over 10 3-CV tests performed in the IID case, using (red) Fed-mv-PPCA and (blue) DP-Fed-mv-PPCA. By model definition, $\widetilde{\sigma}^{(k)^2}$ represents the global variance of the Gaussian noise for the $k^{\textrm{th}}$-view. As one can see, when DP is introduced the estimated global data variance is greater for every view. This fact can affect the performance of the final global model, both for the reconstruction and the separation tasks.}\label{Convergence_global_supp}
\end{figure}

%\begin{figure}[ht]
%\centering
%\includegraphics[scale=0.4]{Figures/ADNI_2Ls_fedprox_3.png}
%\caption{ADNI data, scenario G with 3 centers. MAE and accuracy in the latent space varying the FedProx proximal parameter $\lambda$ using VAE (first column) and mc-VAE (second column). Results corresponding to the FedAvg aggregation scheme and VAE (resp. mc-VAE) with 2 layers (first element) and 1 layer (second element) are further provided for completeness.}\label{fig_adni_fedprox}
%\end{figure}

\begin{figure}[ht]
\centering
     \begin{subfigure}[b]{\textwidth}
         \centerline{\includegraphics[width=.5\textwidth]{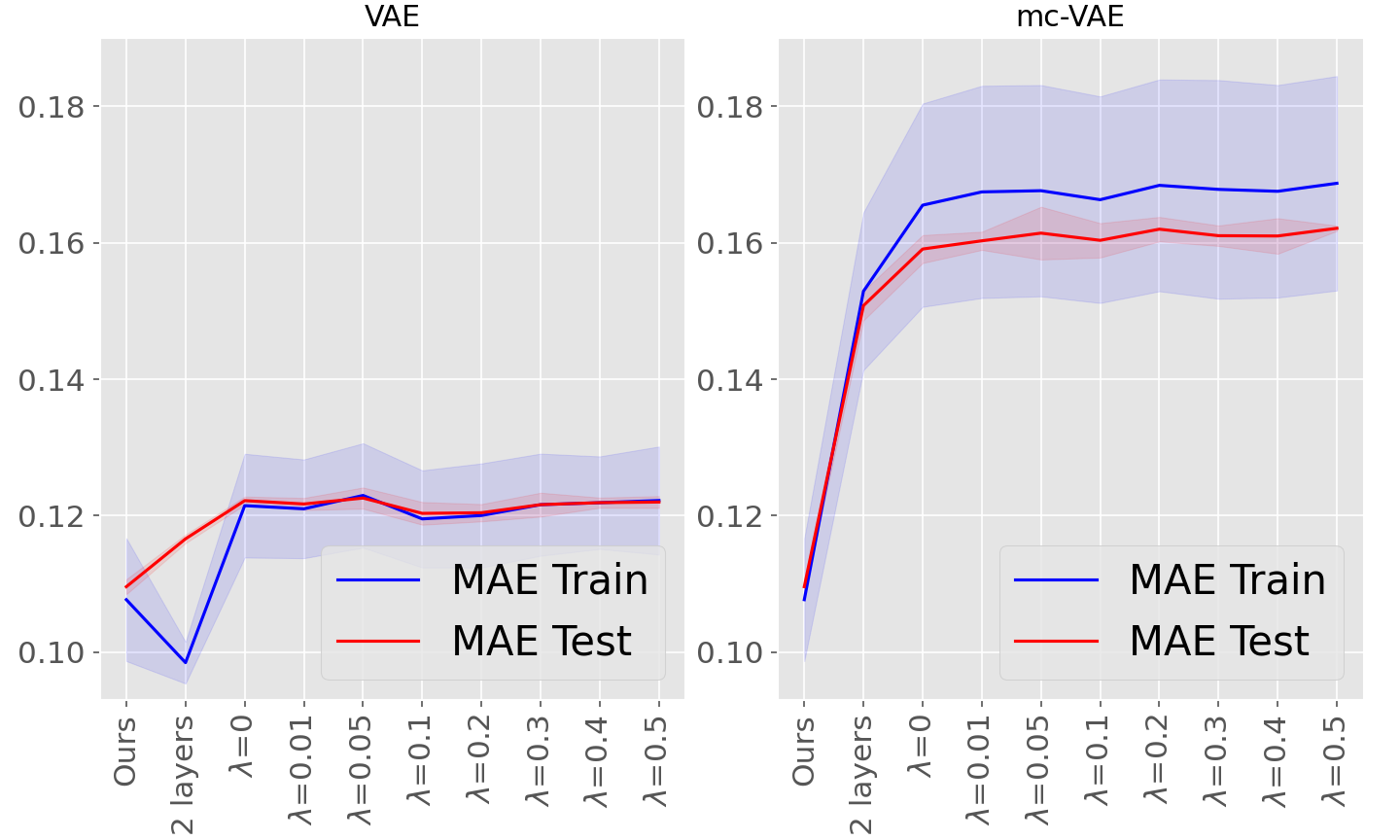}}
         \caption{Mean absolute error for VAE (left) and mc-VAE (right)}
     \end{subfigure}
     \hfill
     \begin{subfigure}[b]{\textwidth}
         \centerline{\includegraphics[width=.5\textwidth]{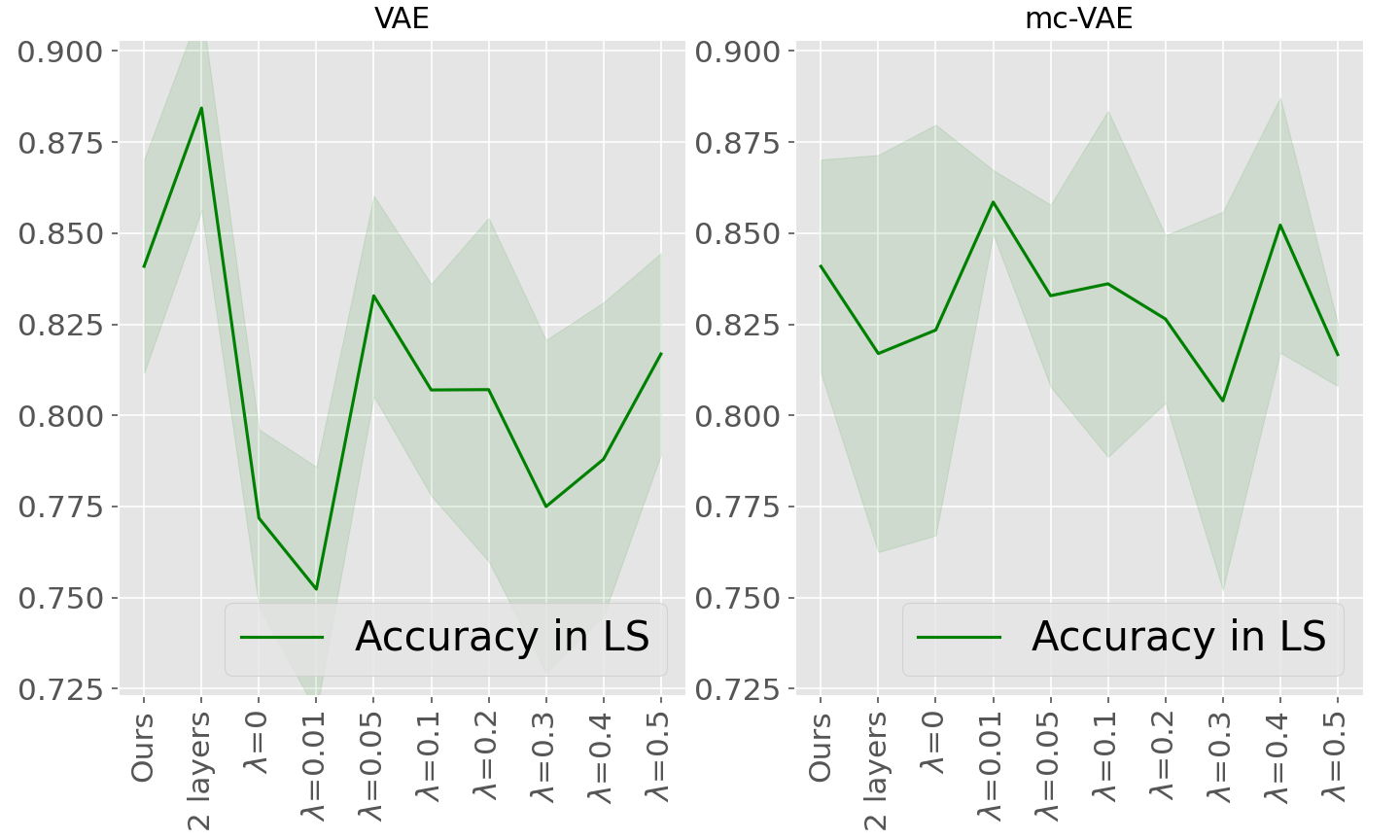}}
         \caption{Classification accuracy in the latent space for VAE (left) and mc-VAE (right)}
     \end{subfigure}
\caption{ADNI data, scenario G with 3 centers, comparing the performance of VAE (left column) and mc-VAE (right column) either increasing the number of considered layers (2 layers), or adopting a robust aggregation scheme alternative to Fed-Avg (FedProx, with varying parameter $\lambda$). For the sake of comparison, in each plot, the first element corresponds to results obtained with our method - Fed-mv-PPCA - for the considered scenario. Upper row: MAE results for the (blue) train and (red) test datasets. Bottom row: accuracy in the latent space for the test dataset. For each metric we provide results obtained using a 2-layers VAE, resp. a 2-layers mc-VAE  (second element of each plot), and federated averaging as aggregation scheme. Finally, results for both MAE and accuracy in the latent space using FedProx as aggregation scheme are provided, with the FedProx proximal parameter $\lambda$ varying between 0 to 0.5. Note that if we set $\lambda=0$ we recover the Fed-Avg aggregation scheme.}\label{fig_adni_fedprox}
\end{figure}

\end{document}